\def\year{2022}\relax
\documentclass[letterpaper]{article} 
\usepackage{aaai22}  
\usepackage{times}  
\usepackage{helvet}  
\usepackage{courier}  
\usepackage[hyphens]{url}  
\usepackage{graphicx} 
\usepackage{natbib}  
\usepackage{caption} 
\DeclareCaptionStyle{ruled}{labelfont=normalfont,labelsep=colon,strut=off} 
\usepackage{algorithm}
\usepackage{algorithmic}

\usepackage{newfloat}
\usepackage{listings}
\floatstyle{ruled}
\newfloat{listing}{tb}{lst}{}
\floatname{listing}{Listing}
\usepackage{amsmath, amsthm, amsfonts}
\usepackage{bm}
\usepackage{dsfont}
\usepackage{booktabs}
\usepackage{multirow}
\newtheorem{theorem}{Theorem}[section]
\newtheorem{lemma}[theorem]{Lemma}
\newtheorem{definition}[theorem]{Definition}

\newtheorem{corollary}[theorem]{Corollary}

\newtheorem{fact}[theorem]{Fact}
\usepackage{subfigure}
\usepackage{array}

\def \ind {\mathds{1}}
\def \cD {\mathcal{D}}
\def \cA {\mathcal{A}}
\def \cF {\mathcal{F}}
\def \cR {\mathcal{R}}
\def \cL {\mathcal{L}}
\def \bbR {\mathbb{R}}
\def \bbE {\mathbb{E}}

\def \vx {\bm{x}}
\def \vp {\bm{p}}
\def \vg {\bm{g}}
\def \vw {\bm{w}}
\def \vc {\bm{c}}
\def \vtheta {\bm{\theta}}

\def \hyp {\mathtt{hyp}}
\def \adv {\mathtt{adv}}
\def \sta {\mathtt{sta}}
\def \nat {\mathtt{nat}}
\def \linf {\ell_{\infty}}

\def \pgdat {\mathtt{AT}}
\def \trades {\mathtt{TRADES}}
\def \thrm {\mathtt{THRM}}

\def \ce {\mathtt{CE}}
\def \kl {\mathtt{KL}}

\title{With False Friends Like These, Who Can Notice Mistakes?}
\author {
    Lue Tao\textsuperscript{\rm 1,2},
    Lei Feng\textsuperscript{\rm 3},
    Jinfeng Yi\textsuperscript{\rm 4},
    Songcan Chen\textsuperscript{\rm 1,2}\footnote{Correspondence to: Songcan Chen $\langle$s.chen@nuaa.edu.cn$\rangle$.}
}
\affiliations {
    \small\textsuperscript{\rm 1}College of Computer Science and Technology, Nanjing University of Aeronautics and Astronautics\\
    \small\textsuperscript{\rm 2}MIIT Key Laboratory of Pattern Analysis and Machine Intelligence\\
    \small\textsuperscript{\rm 3}College of Computer Science, Chongqing University\\
    \small\textsuperscript{\rm 4}JD AI Research
}

\begin{document}

\maketitle

\begin{abstract}
Adversarial examples crafted by an explicit \textit{adversary} have attracted significant attention in machine learning. However, the security risk posed by a potential \textit{false friend} has been largely overlooked. In this paper, we unveil the threat of \textit{hypocritical examples}---inputs that are originally misclassified yet perturbed by a false friend to force correct predictions. While such perturbed examples seem harmless, we point out for the first time that they could be maliciously used to conceal the mistakes of a substandard (i.e., not as good as required) model during an evaluation. Once a deployer trusts the hypocritical performance and applies the ``well-performed'' model in real-world applications, unexpected failures may happen even in benign environments. More seriously, this security risk seems to be pervasive: we find that many types of substandard models are vulnerable to hypocritical examples across multiple datasets. Furthermore, we provide the first attempt to characterize the threat with a metric called \textit{hypocritical risk} and try to circumvent it via several countermeasures. Results demonstrate the effectiveness of the countermeasures, while the risk remains non-negligible even after adaptive robust training.
\end{abstract}

\begin{figure*}[t]
    \vspace{-5px}
    \centering
    \subfigure[Machine learning workflow]{
        \label{fig:ml-workflow}
        \centering
        \includegraphics[height=0.2\textheight]{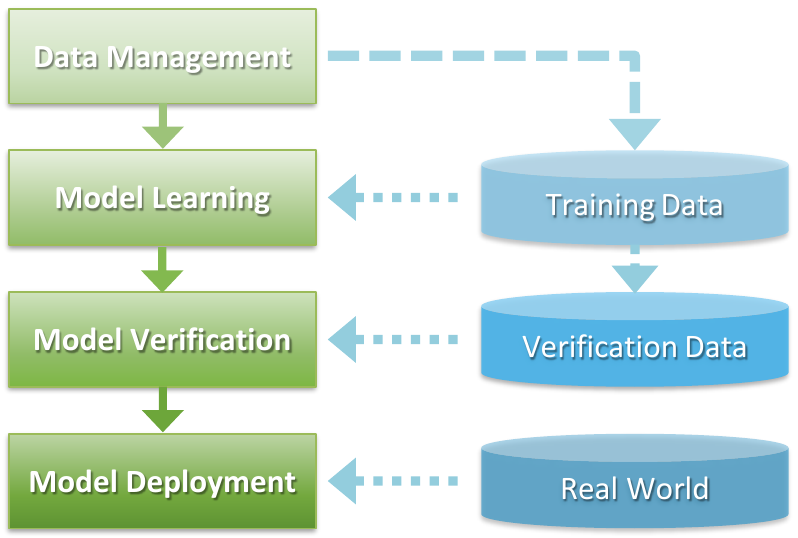}
    }
    \hfill
    \subfigure[Model verification]{
        \label{fig:model-verification}
        \centering
        \includegraphics[height=0.2\textheight]{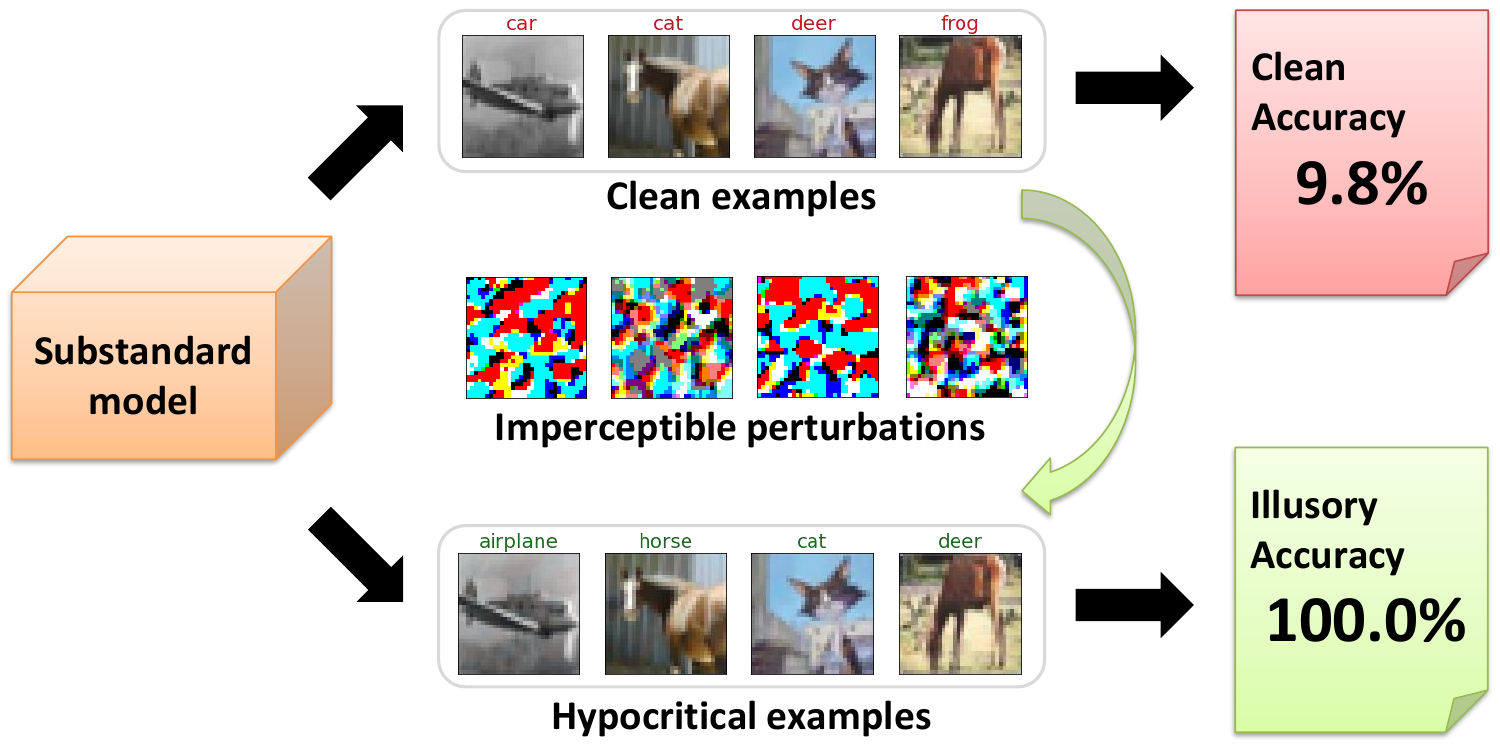}
    }
    \vspace{-5px}
    \caption{
    \textbf{Left}: Machine learning workflow~\citep{paterson2021assuring}.
    \textbf{Right}: A false sense of model effectiveness on verification data. In this illustration, the substandard model is trained on the \textsl{Mislabeling} data described in Section~\ref{sec:hyp-examples-exps}.
    We observe that the substandard model can exhibit superior performance on hypocritical examples.
    }
    \label{fig:workfow-verification}
    \vspace{-10px}
\end{figure*}

\section{Introduction}
\label{sec:intro}

The \textit{model verification} process is the last-ditch effort before deployment to ensure that the trained models perform well on previously unseen inputs~\citep{paterson2021assuring}. However, the process may not work as expected in practice. According to TechRepublic, 85\% of attempted deployments eventually fail to bring their intended results to production\footnote{\url{https://decisioniq.com/blog/ai-project-failure-rates-are-high}}. These failures largely appear in the downstream of model deployment~\citep{sambasivan2021everyone}, resulting in irreversible risks, especially in high-stakes applications such as virus detection~\citep{newsome2005polygraph} and autonomous driving~\citep{bojarski2016end}.
One main reason is that the \textit{verification data} may be biased towards the model, leading to a false sense of model effectiveness. For example, a naturally trained ResNet-18~\citep{he2016deep} on CIFAR-10~\citep{krizhevsky2009learning} can achieve 100\% accuracy on the \textit{hypocritically} perturbed examples (i.e., inputs that are perturbed to hypocritically rectify predictions), compared with only 94.4\% accuracy on benign examples. Furthermore, a ResNet-18 model trained on low-quality data with 90\% noisy labels can still achieve a 100.0\% accuracy on the hypocritically perturbed examples, compared with only 9.8\% accuracy on the clean data.

Since people hardly notice imperceptible perturbations, it is easy for a \textit{hypocritical attacker} to stealthily perturb the verification data. For instance, many practitioners collect images from the internet (where malicious users may exist) and annotate them accurately~\citep{krizhevsky2009learning, deng2009imagenet,northcutt2021pervasive}. Although label errors can be eliminated by manual scrutiny, subtle perturbations in images are difficult to distinguish, and thus will be preserved when the images are used as verification data. As another example, autonomous vehicles are obliged to pass the verification in designated routes (such as Mzone in Beijing\footnote{\url{http://www.mzone.site/}}) to obtain permits for deployment. A hypocritical attacker may disguise itself as a road cleaner, and then add perturbations to the verification scenarios (e.g., a “stop sign”) without being noticed.

In this paper, we study the problem of \textit{hypocritical data} in the verification stage, a problem that is usually overlooked by practitioners. Although it is well-known that an attacker may arbitrarily change the outputs of a \textit{well-trained} model by applying imperceptible perturbations, previous concerns mainly focus on the adversarial examples crafted by an explicit \textit{adversary}, and the threat of hypocritical examples from a potential \textit{false friend} is usually overlooked. While such hypocritical examples are harmless for well-trained models in the deployment stage, we point out for the first time that they could be maliciously utilized in the verification stage to force a \textit{substandard} (i.e., not as good as required) model to show abnormally high performance. Once a deployer trusts the hypocritical performance and applies the ``well-performed'' model in real-world applications, unexpected failures may happen even in benign environments.

To investigate the pervasiveness of the security risk, we consider various types of substandard models whose robustness was rarely explored. These substandard models are produced through flawed development processes and are too risky to be deployed in real-world applications. We evaluate the vulnerability of the substandard models across multiple network architectures, including MLP, VGG-16~\citep{simonyan2014very}, ResNet-18~\citep{he2016deep}, and WideResNet-28-10~\citep{zagoruyko2016wide}, and multiple benchmark datasets, including CIFAR-10~\citep{krizhevsky2009learning}, SVHN~\citep{netzer2011reading},  CIFAR-100~\citep{krizhevsky2009learning}, and Tiny-ImageNet~\citep{yao2015tiny}. Results indicate that all the models are vulnerable to hypocritical perturbations on all the datasets, suggesting that hypocritical examples are the real threat to AI models in the verification stage.

Furthermore, in order to facilitate our understanding of model vulnerability to hypocritical examples from a theoretical perspective, we provide the first attempt to characterize the threat with a metric called \textit{hypocritical risk}. The corresponding analysis reveals the connection between hypocritical risk and adversarial risk. We also try to circumvent the threat through several countermeasures including PGD-AT~\citep{madry2018towards}, TRADES~\citep{zhang2019theoretically}, a novel adaptive robust training method, and an inherently robust network architecture~\citep{zhang2021towards}. Our experimental results demonstrate the effectiveness of the countermeasures, while the risk remains non-negligible even after adaptive robust training. Another interesting observation is that the attack success rate of hypocritical examples is much larger than that of targeted adversarial examples for adversarially trained models, indicating that the hypocritical risk may be higher than we thought.

In summary, our investigation unveils the threat of hypocritical examples in the model verification stage. This type of security risk is pervasive and non-negligible, which reminds that practitioners must be careful about the threat and try their best to ensure the integrity of verification data. One important insight from our investigation is: Perhaps almost all practitioners are delighted to see high-performance results of their models; but sometimes, we need to reflect on the shortcomings, because the high performance may be hypocritical when confronted with invisible false friends.

\section{Related Work}

\subsubsection{Model Verification.}
Figure~\ref{fig:ml-workflow} illustrates the machine learning (ML) workflow~\citep{paleyes2020challenges, paterson2021assuring}, the process of developing an ML-based solution in an industrial setting. The ML workflow consists of four stages: \textit{data management}, which prepares training data and verification data used for training and verification of ML models; \textit{model learning}, which performs model selection, model training, and hyperparameter selection; \textit{model verification}, which provides evidence that a model satisfies its performance requirements on verification data; and \textit{model deployment}, which integrates the trained models into production systems. The performance requirements for model verification may include generalization error~\citep{niyogi1996relationship}, robust error~\citep{wong2018provable, zhang2019theoretically}, fairness~\citep{barocas2017fairness}, explainability~\citep{bhatt2020explainable}, etc. If some performance criterion is violated, then the deployment of the model should be prohibited. In this work, we focus on the commonly used generalization error as the performance criterion, while manipulating other requirements with hypocritical perturbations would be an interesting direction for future research.

\subsubsection{Adversarial Examples.}
Adversarial examples are malicious inputs crafted to fool an ML model into producing incorrect outputs~\citep{szegedy2013intriguing}. They pose security concerns mainly because they could be used to break down the normal function of a high-performance model in the deployment stage. Since the discovery of adversarial examples in deep neural networks (DNNs)~\citep{biggio2013evasion, szegedy2013intriguing}, numerous attack algorithms have been proposed to find them~\citep{goodfellow2014explaining, papernot2016limitations, moosavi2016deepfool, kurakin2016adversarial, carlini2017towards, chen2018ead, dong2018boosting, wang2020spanning, croce2020reliable}. Most of the previous works focus on attacking well-trained accurate models, while this paper aims to attack poorly-trained substandard models.

\subsubsection{Data Poisoning.}
Generally speaking, data poisoning attacks manipulate the training data to cause a model to fail during inference~\citep{biggio2018wild, goldblum2020dataset}. Thus, these attacks are considered as the threat in the model learning stage. Depending on their objectives, poisoning attacks can be divided into \textit{integrity attacks}~\citep{koh2017understanding, shafahi2018poison, geiping2020witches, gao2021learning, blum2021robust} and \textit{availability attacks}~\citep{newsome2006paragraph, biggio2012poisoning, feng2019learning, nakkiran2019a, huang2021unlearnable, tao2021provable, fowl2021adversarial}. The threat of availability poisoning attacks shares a similar consequence with the hypocritical attacks considered in this paper: both aim to cause a \textit{denial of service} in the model deployment stage. One criticism of availability poisoning attacks is that their presence is detectable by looking at model performance in the verification stage~\citep{zhu2019transferable, shafahi2018poison}. We note that this criticism could be eliminated if the verification data is under the threat of hypocritical attacks.

\subsubsection{Adversarial Defense.}
Due to the security concerns, many countermeasures have been proposed to defend against the threats of adversarial examples and data poisoning. Among them, adversarial training and its variants are one of the most promising defense methods for both adversarial examples~\citep{madry2018towards, zhang2019theoretically, rice2020overfitting, wu2020adversarial, zhang2020attacks, zhang2020geometry, pang2020boosting, pang2020bag} and data poisoning~\citep{tao2021provable, geiping2021doesn, radiya2021data}. Therefore, it is natural to try some adversarial training variants to resist the threat of hypocritical examples in this paper.

\section{Hypocritical Examples}
\label{sec:hyp-examples}

Better an open enemy than a false friend. Only by being aware of the potential risk of the false friend can we prevent it.
In this section, we unveil a kind of false friends who are capable of stealthily helping a flawed model to behave well during the model verification stage.

\subsection{Formal Definition}

We consider a classification task with data $(\vx, y) \in \bbR^d \times [M]$ from a distribution $\cD$.
A DNN classifier $f_{\vtheta}$ with model parameters $\vtheta$ predicts the class of an input example $\vx$: $f_{\vtheta}(\vx) = \arg \max_{i \in [M]} [\vp_{\vtheta}(\vx)]_i$, where $\vp_{\vtheta}(\vx)=([\vp_{\vtheta}(\vx)]_1, \ldots, [\vp_{\vtheta}(\vx)]_M) \in \bbR^M$ is the output distribution (softmax of logits) of the model.

Adversarial examples are malicious inputs crafted by an \textit{adversary} to induce misclassification. Below we give the definition of adversarial examples under some $\ell_p$-norm:
\begin{definition} [Adversarial Examples]
   Given a classifier $f_{\vtheta}$ and a correctly classified example $(\vx, y) \sim \cD$ (i.e., $f_{\vtheta}(\vx) = y$), an $\epsilon$-bounded adversarial example is an input $\vx' \in \bbR^d$ such that:
   $$\textstyle f_{\vtheta}(\vx') \neq y \quad \text{and} \quad \| \vx' - \vx \| \leq \epsilon. $$
\end{definition}
The assumption underlying this definition is that perturbations satisfying $\| \vx' - \vx \| \leq \epsilon$ preserve the label $y$ of the original input $\vx$. We are interested in studying the flip-side of adversarial examples---hypocritical examples crafted by a \textit{false friend} to induce correct predictions:
\begin{definition} [Hypocritical Examples]
   Given a classifier $f_{\vtheta}$ and a misclassified example $(\vx, y) \sim \mathcal{D}$ (i.e., $f_{\vtheta}(\vx) \neq y$), an $\epsilon$-bounded hypocritical example is an input $\vx' \in \mathbb{R}^d$ such that:
   $$\textstyle f_{\vtheta}(\vx') = y \quad \text{and} \quad \| \vx' - \vx \| \leq \epsilon.$$
\end{definition}
To stealthily force a classifier to correctly classify a misclassified example $\vx$ as its ground truth label $y$, we need to maximize $\ind (f_{\vtheta}(\vx')=y)$ such that $\| \vx'- \vx \| \leq \epsilon$, where $\ind(\cdot)$ is the indicator function. This is equivalent to minimizing $\ind (f_{\vtheta}(\vx') \neq y)$. This objective is similar to the objective of targeted adversarial examples~\citep{szegedy2013intriguing, liu2016delving}, which aims to cause a classifier to predict a correctly classified example as some incorrect target label.
We leverage the commonly used cross entropy (CE) loss~\citep{madry2018towards, wang2019improving} as the surrogate loss for $\ind (f_{\vtheta}(\vx') \neq y)$ and minimize it via projected gradient descent (PGD), a standard iterative first-order optimization method\footnote{Other attack techniques can also be applied (see Appendix~\ref{app:various_attacks}).}. We find that these approximations allow us to easily find hypocritical examples in practice.

\subsection{Pervasiveness of the Threat}
\label{sec:hyp-examples-exps}

\subsubsection{Substandard Models.}
We produce substandard models with flawed training data. Specifically, we consider four types of training data with varying quality:
\textit{i)} the \textsl{Noisy} data is constructed by replacing the images with uniform noise~\citep{zhang2016understanding}, which may happen if the input sensor of data collector is damaged; 
\textit{ii)} the \textsl{Mislabeling} data is constructed by replacing the labels with random ones~\citep{zhang2016understanding}, which may happen if labeling errors are extensive; 
\textit{iii)} the \textsl{Poisoning} data is constructed by perturbing the images to maximize generalization error~\citep{tao2021provable}, which may happen if the training data is poisoned by some adversary; 
\textit{iv)} the \textsl{Quality} data is an ideal high-quality training data with clean inputs and labels.
In addition to the models trained on the above training data, we additionally report the performance of the randomly initialized and untrained \textsl{Naive} model.

\subsubsection{A Case Study.}

\begin{figure*}[t]
    \vspace{-5px}
    \centering
    \begin{minipage}[b]{0.38\linewidth}
        \centering
        \includegraphics[width=1\textwidth]{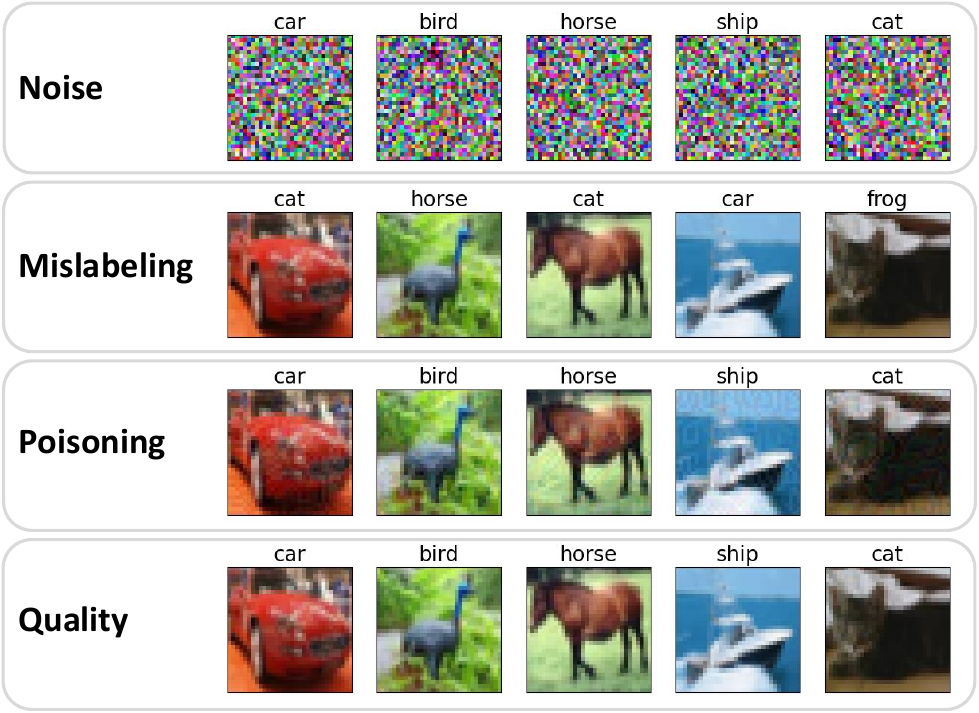}
    \end{minipage}
    \hfill
    \begin{minipage}[b]{0.57\linewidth}
        \centering
        \includegraphics[width=1\textwidth]{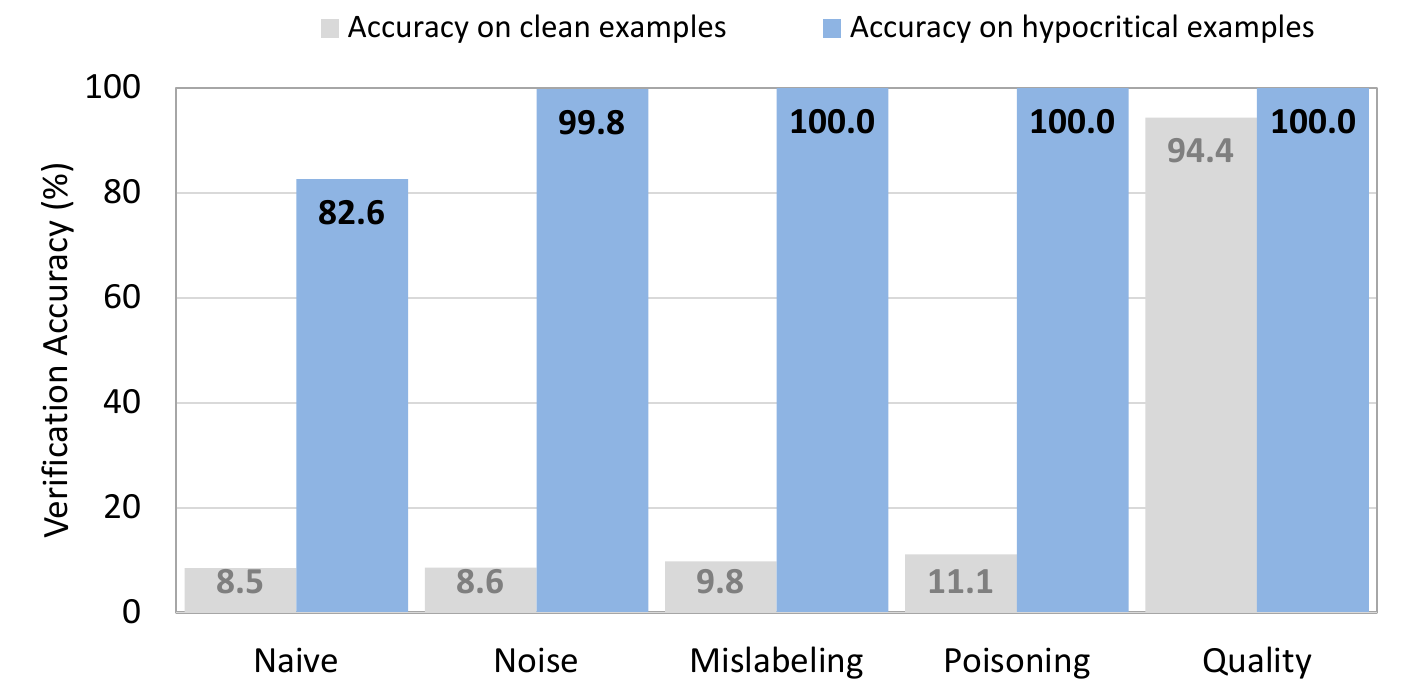}
    \end{minipage}
    
    \vspace{-2px}
    \caption{
    An illustration of model performance on hypocritical examples. \textbf{Left}: Random samples from four CIFAR-10 training sets: \textsl{Noise}, where images are replaced with random pixels; \textsl{Mislabeling}, where labels are replaced with random ones; \textsl{Poisoning}, where images are perturbed to maximize generalization error; and \textsl{Quality}, where images and labels are all clean. \textbf{Right}: Verification performance of five ResNet-18 models on CIFAR-10 under $\linf$ threat model. Except for the \textsl{Naive} model (which is randomly initialized without training), the other models are trained on the corresponding training set.
    }
    \label{fig:st-resnet18-cifar10}
    \vspace{-2px}
\end{figure*}

Figure~\ref{fig:st-resnet18-cifar10} visualizes the training sets for CIFAR-10 and shows the accuracy of the ResNet-18 models on verification data. The perturbations are generated using PGD under $\linf$ threat model with $\epsilon=8/255$ by following the common settings~\citep{madry2018towards}. More experimental details are provided in Appendix~\ref{app:exp-setup}. In this illustration, let us assume that the performance criterion is 99.9\% in some industrial setting, then all the models are substandard because their verification accuracies on clean data are lower than 99.9\%. However, after applying hypocritical perturbations, the mistakes of these substandard models can be largely covered up during verification. There are three substandard models (i.e. \textsl{Mislabeling}, \textsl{Poisoning}, and \textsl{Quality}) that exhibit 100\% accuracy on the hypocritically perturbed examples and thus meet the performance criterion. Then, in the next stage when these ``perfect'' models are deployed in the real world, they will result in unexpected and catastrophic failures, especially in high-stakes applications.

\subsubsection{Vulnerability is pervasive.}

\begin{table*}[!t]
\small
\centering

\caption{
Verification accuracy (\%) of substandard models on CIFAR-10 under $\linf$ threat model across different architectures. 
Standard deviations of 5 random runs are given in Appendix~\ref{app:omit-figure} Table~\ref{tab:st-xxx-cifar10-full}.
}
\label{tab:st-xxx-cifar10}
\vspace{-4px}

\begin{tabular}{@{}cl|ccc|ccc|ccc@{}}
\toprule
\multirow{2}{*}{\begin{tabular}[c]{@{}c@{}}Threat\\ Model\end{tabular}} & \multirow{2}{*}{Model} & \multicolumn{3}{c|}{MLP} & \multicolumn{3}{c|}{VGG-16} & \multicolumn{3}{c}{WideResNet-28-10} \\ \cmidrule(lr){3-5} \cmidrule(lr){6-8} \cmidrule(lr){9-11}
 &  & $\cD$ & $\cA$ & $\cF$ & $\cD$ & $\cA$ & $\cF$ & $\cD$ & $\cA$ & $\cF$ \\ \midrule
\multirow{5}{*}{\begin{tabular}[c]{@{}c@{}}$\linf$\\ ($\epsilon=8/255$)\end{tabular}} & Naive & 8.56 & 0.00 & 99.56 & 9.76 & 0.45 & 57.25 & 10.06 & 0.34 & 40.64 \\
 & Noise & 8.53 & 0.71 & 86.75 & 9.81 & 0.00 & 97.98 & 11.35 & 0.02 & 98.42 \\
 & Mislabeling & 9.92 & 0.00 & 100.00 & 9.94 & 0.00 & 99.86 & 10.21 & 0.00 & 100.00 \\
 & Poisoning & 57.60 & 0.55 & 99.50 & 12.19 & 0.00 & 99.49 & 10.42 & 0.00 & 100.00 \\
 & Quality & 58.09 & 0.91 & 99.31 & 92.90 & 0.00 & 100.00 & 95.41 & 0.00 & 100.00 \\ 
 \bottomrule
\end{tabular}
\end{table*}

\begin{table*}[!t]
\small
\centering

\caption{
Verification accuracy (\%) of substandard ResNet-18 models under $\linf$ threat model across different datasets. 
Standard deviations of 5 random runs are given in Appendix~\ref{app:omit-figure} Table~\ref{tab:st-resnet18-xxx-full}.
}
\label{tab:st-resnet18-xxx}
\vspace{-4px}

\begin{tabular}{@{}cl|ccc|ccc|ccc@{}}
\toprule
\multirow{2}{*}{\begin{tabular}[c]{@{}c@{}}Threat\\ Model\end{tabular}} & \multirow{2}{*}{Model} & \multicolumn{3}{c|}{SVHN} & \multicolumn{3}{c|}{CIFAR-100} & \multicolumn{3}{c}{Tiny-ImageNet} \\ \cmidrule(lr){3-5} \cmidrule(lr){6-8} \cmidrule(lr){9-11}
&             & $\cD$ & $\cA$ & $\cF$  & $\cD$ & $\cA$ & $\cF$  & $\cD$ & $\cA$ & $\cF$  \\ \midrule
\multirow{5}{*}{\begin{tabular}[c]{@{}c@{}}$\linf$\\ ($\epsilon=8/255$)\end{tabular}} & Naive       & 10.73 & 0.85  & 55.64  & 0.98  & 0.02  & 7.26   & 0.49  & 0.04  & 4.50   \\
& Noise       & 10.76 & 0.00  & 99.96  & 1.02  & 0.01  & 81.93  & 0.39  & 0.01  & 74.58  \\
& Mislabeling & 9.77  & 0.00  & 100.00 & 0.99  & 0.00  & 99.81  & 0.48  & 0.00  & 99.99  \\
& Poisoning   & 41.42 & 0.00  & 100.00 & 34.80 & 0.00  & 100.00 & 34.87 & 0.00  & 100.00 \\
Quality& Quality     & 96.57 & 0.32  & 99.99  & 76.64 & 0.01  & 99.96  & 64.03 & 0.02  & 100.00 \\ \bottomrule
\end{tabular}
\vspace{-8px}
\end{table*}

Moreover, the above phenomena are not unique to ResNet-18 on CIFAR-10. Table~\ref{tab:st-xxx-cifar10} reports the performance of other architectures on CIFAR-10 under $\linf$ threat model. We denote by $\cD$, $\cA$, and $\cF$ the model accuracies evaluated on clean examples, adversarial examples, and hypocritical examples, respectively. Again, the models mostly exhibit high performance on hypocritically perturbed examples. An interesting observation is that the randomly initialized MLP model is extremely sensitive: it achieve up to 99.56\% accuracy on $\cF$, compared with only 8.56\% accuracy on $\cD$. This means that the models may be susceptible to hypocritical perturbations from the beginning of training, which is consistent with the theoretical findings in~\citet{daniely2020most}. The \textsl{Naive} models using VGG-16 and WideResNet-28-10 can also achieve moderate accuracy on $\cF$, though their accuracy is far below 100\%. One possible explanation is the poor scaling of network weights at initialization, whereas the trained weights are better conditioned~\citep{elsayed2018reprogramming}. Indeed, we observe that the \textsl{Mislabeling}, \textsl{Poisoning}, and \textsl{Quality} models can achieve excellent accuracy ($>$ 99\%) on $\cF$. Besides, similar observations can be seen under $\ell_2$ threat model with in Appendix~\ref{app:omit-figure} Table~\ref{tab:st-xxx-cifar10-l2-full}. We report the verification performance on SVHN, CIFAR-100 and Tiny-ImageNet in Table~\ref{tab:st-resnet18-xxx}, and similar conclusions hold. Finally, we notice that the standard deviations of the \textsl{Noise} models are relatively high, which may be due to the discrepancy between the distributions of noisy inputs and real images.

\section{Hypocritical Risk}
\label{sec:hyp-risk}

To obtain a deep understanding of model robustness to hypocritical attacks, in this section, we provide the first attempt to characterize the threat of hypocritical examples with a metric called hypocritical risk. Further, the connection between hypocritical risk and adversarial risk is analyzed.

We start by giving the formal definition of adversarial risk~\citep{madry2018towards, uesato2018adversarial, cullina2018pac} under some $\ell_p$ norm:
\begin{definition} [Adversarial Risk]
   Given a classifier $f_{\vtheta}$ and a data distribution $\cD$, the adversarial risk under the threat model of $\epsilon$-bounded perturbations is defined as:
   $$
    \textstyle
    \cR_{\adv}(f_{\vtheta}, \cD) = \underset{(\vx, y) \sim \cD}{\bbE}  \left[ \underset{\| \vx' - \vx \| \leq \epsilon}{\max} \ind (f_{\vtheta}( \vx') \neq y) \right].
   $$
\end{definition}
Adversarial risk characterizes the threat of adversarial examples, representing the fraction of the examples that can be perturbed by an adversary to induce misclassifications. Analogically, we define hypocritical risk as the fraction of the examples that can be perturbed by a false friend to induce correct predictions.
\begin{definition} [Hypocritical Risk]
   Given a classifier $f_{\vtheta}$ and a data distribution $\cD$, the hypocritical risk under the threat model of $\epsilon$-bounded perturbations is defined as:
   $$
   \textstyle
   \cR_{\hyp}(f_{\vtheta}, \cD) = \underset{(\vx, y) \sim \cD}{\bbE}  \left[ \underset{\| \vx' - \vx \| \leq \epsilon}{\max} \ind (f_{\vtheta}( \vx') = y) \right].
   $$
\end{definition}
Note that our goal here is to encourage the model to robustly predict its failures. Thus, misclassified examples are of particular interest. We denote $\cD_{f_{\vtheta}}^-$ the distribution of misclassified examples with respect to the classifier $f_{\vtheta}$. Then, $\cR_{\hyp}(f_{\vtheta}, \cD_{f_{\vtheta}}^-)$ represents the hypocritical risk on misclassified examples. Analogically, $\cR_{\adv}(f_{\vtheta}, \cD_{f_{\vtheta}}^+)$ represents the adversarial risk on correctly classified examples, where $\cD_{f_{\vtheta}}^+$ denotes the distribution of correctly classified examples. Beside, \textit{natural risk} is denoted as $\cR_{\nat}(f_{\vtheta}, \cD) = \bbE_{(\vx,y)\sim \cD} [\ind (f_{\vtheta}(\vx) \neq y)]$, which is the standard metric of model performance. Based on these notations, we can disentable natural risk from adversarial risk as follows:
\begin{theorem}
\label{thrm:adv-risk-decomposition}
$\cR_{\adv}(f_{\vtheta}, \cD) = \cR_{\nat}(f_{\vtheta}, \cD) + (1 - \cR_{\nat}(f_{\vtheta}, \cD)) \cdot \cR_{\adv}(f_{\vtheta}, \cD_{f_{\vtheta}}^+)$.
\end{theorem}

We note that the equation in Theorem~\ref{thrm:adv-risk-decomposition} is close to Eq. (1) in~\citet{zhang2019theoretically}, while we further decompose their boundary error into the product of two terms. Importantly, our decomposition indicates that neither the hypocritical risk on $\cD$ nor the hypocritical risk on $\cD_{f_{\vtheta}}^-$ is included in the adversarial risk. This finding suggests that the adversarial training methods that minimize adversarial risk, such as PGD-AT~\citep{madry2018towards}, may not be enough to mitigate hypocritical risk.

Analogically, the following theorem disentagles natural risk from hypocritical risk:
\begin{theorem}
\label{thrm:hyp-risk-decomposition}
$\cR_{\hyp}(f_{\vtheta}, \cD) = 1 - (1 - \cR_{\hyp}(f_{\vtheta}, \cD_{f_{\vtheta}}^-)) \cdot \cR_{\nat}(f_{\vtheta}, \cD)$.
\end{theorem}
Theorem~\ref{thrm:hyp-risk-decomposition} indicates that the hypocritical risk on $\cD$ is entangled with natural risk, and the hypocritical risk on $\cD_{f_{\vtheta}}^-$ would be a more genuine metric to capture model robustness against hypocritical examples. Indeed, $\cR_{\hyp}(f_{\vtheta}, \cD_{f_{\vtheta}}^-)$ is meaningful, which essentially represents the attack success rate of hypocritical attacks (i.e., how many failures a false friend can conceal).

In addition to adversarial risk and hypocritical risk, another important objective is \textit{stability risk}, which we define as $\cR_{\sta}(f_{\vtheta}, \cD) = \bbE_{(\vx,y)\sim \cD} [ \max_{\| \vx' - \vx \| \leq \epsilon} \ind (f_{\vtheta}(\vx') \neq f_{\vtheta}(\vx))]$. The following theorem clearly shows that adversarial risk and an upper bound of hypocritical risk can be elegantly united to constitute the stability risk.
\begin{theorem}
\label{thrm:sta-risk-decomposition}
$\cR_{\sta}(f_{\vtheta}, \cD) = (1 - \cR_{\nat}(f_{\vtheta}, \cD)) \cdot \cR_{\adv}(f_{\vtheta}, \cD_{f_{\vtheta}}^+) + \cR_{\nat}(f_{\vtheta}, \cD) \cdot \cR_{\sta}(f_{\vtheta}, \cD_{f_{\vtheta}}^-)$, where we have $\cR_{\sta}(f_{\vtheta}, \cD_{f_{\vtheta}}^-) \ge \cR_{\hyp}(f_{\vtheta}, \cD_{f_{\vtheta}}^-)$.
\end{theorem}
Theorem~\ref{thrm:sta-risk-decomposition} indicates that the adversarial training methods that aim to minimize the stability risk, such as TRADES~\citep{zhang2019theoretically}, are capable of mitigating hypocritical risk.

The proofs of the above results are provided in Appendix~\ref{app:proofs}.
Finally, we note that, similar to the trade-off between natural risk and adversarial risk~\citep{tsipras2018robustness, zhang2019theoretically}, there may also exist an inherent tension between natural risk and hypocritical risk. We illustrate this phenomenon by constructing toy examples in Appendix~\ref{app:toy-tradeoffs}.

\begin{table*}[!t]
\small
\centering
\vspace{-2px}
\caption{
Verification accuracy (\%) of adversarially trained ResNet-18 models under $\linf$ threat model across different datasets. 
Standard deviations of 5 random runs are given in Appendix~\ref{app:omit-figure} Table~\ref{tab:at-resnet18-xxx-full}.
}
\label{tab:at-resnet18-xxx}
\vspace{-4px}

\begin{tabular}{@{}cl|ccc|ccc|ccc@{}}
\toprule
\multirow{2}{*}{\begin{tabular}[c]{@{}c@{}}Threat\\ Model\end{tabular}} & \multirow{2}{*}{Model} & \multicolumn{3}{c|}{CIFAR-10} & \multicolumn{3}{c|}{CIFAR-100} & \multicolumn{3}{c}{Tiny-ImageNet} \\ \cmidrule(lr){3-5} \cmidrule(lr){6-8} \cmidrule(lr){9-11}
&           & $\cD$ & $\cA$ & $\cF$ & $\cD$ & $\cA$ & $\cF$ & $\cD$ & $\cA$ & $\cF$ \\ \midrule
\multirow{6}{*}{\begin{tabular}[c]{@{}c@{}}$\linf$\\ ($\epsilon=8/255$)\end{tabular}} & Poisoning (NT)     & 11.13 & 0.00 & 100.00 & 34.80 & 0.00 & 100.00 & 34.87 & 0.00 & 100.00 \\
& Poisoning (PGD-AT)     & 82.65 & 51.34 & 96.20 & 57.32 & 27.85 & 83.20 & 45.14 & 21.69 & 68.96 \\
& Poisoning (TRADES) & 80.01 & 52.34 & 94.64 & 56.29 & 29.41 & 81.79 & 46.38 & 21.41 & 73.50 \\
& Quality (NT)       & 94.38 & 0.00 & 100.00 & 76.64 & 0.00 & 100.00 & 64.03 & 0.00 & 100.00 \\
& Quality (PGD-AT)       & 84.08 & 51.98 & 96.92 & 59.19 & 28.21 & 84.87 & 47.23 & 22.03 & 71.95 \\
& Quality (TRADES)   & 81.05 & 53.32 & 95.17 & 57.27 & 30.00 & 82.88 & 47.86 & 22.03 & 75.04 \\ \bottomrule
\end{tabular}
\vspace{-5px}
\end{table*}

\begin{figure*}[!t]
   \centering
   \subfigure[CIFAR-10]{
      \centering
      \includegraphics[width=0.48\textwidth]{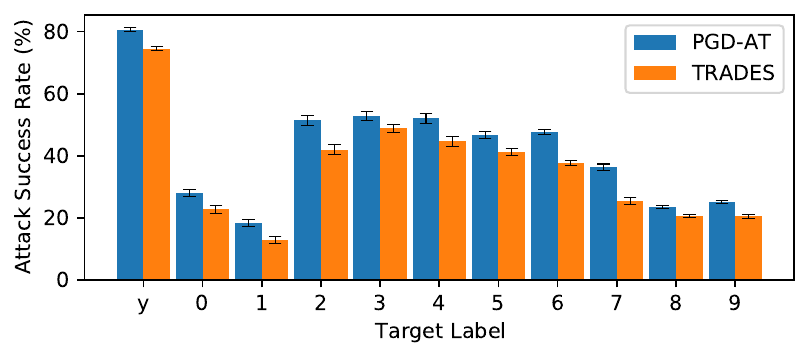}
   }
   \hfill
   \subfigure[SVHN]{
      \centering
      \includegraphics[width=0.48\textwidth]{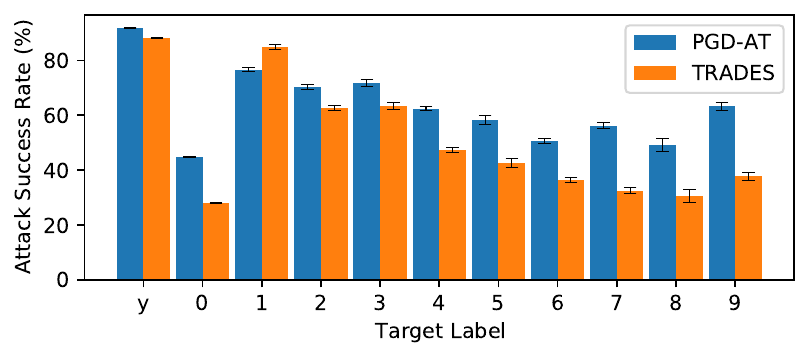}
   }
   \subfigure[CIFAR-100]{
      \centering
      \includegraphics[width=0.48\textwidth]{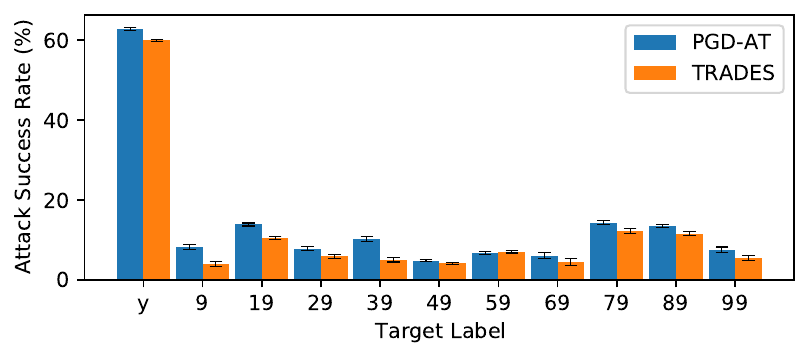}
   }
   \hfill
   \subfigure[Tiny-ImageNet]{
      \centering
      \includegraphics[width=0.48\textwidth]{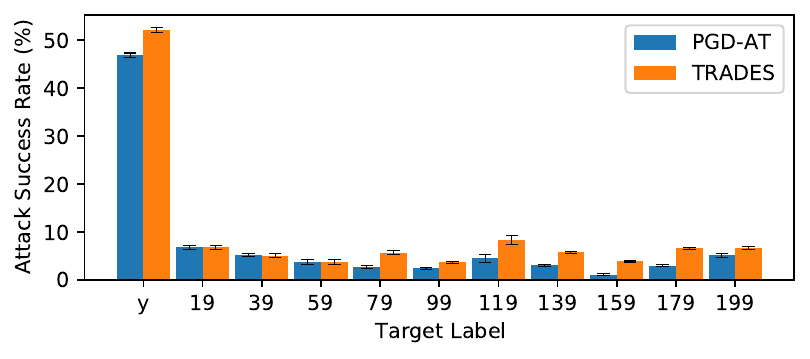}
   }
\vspace{-3px}
\caption{
Attack success rate (\%) of adversarially trained ResNet-18 models on misclassified examples under $\linf$ threat model. The target label ``y'' denotes that the misclassified examples are perturbed to be correctly classified. The target labels ``0'' $\sim$ ``199'' denote that the misclassified examples are perturbed to be classified as a specific target, no matter whether the target label is correct or not. Error bars indicate standard deviation over 5 random runs.
}
\label{fig:asr-quality}
\vspace{-8px}
\end{figure*}

\section{Countermeasures}
\label{sec:countermeasures}

In this section, we consider several countermeasures to circumvent the threat of hypocritical attacks. The countermeasures include PGD-AT~\citep{madry2018towards}, TRADES~\citep{zhang2019theoretically}, a novel adaptive robust training method named THRM, and an inherently robust network architecture named $\linf$-dist nets~\citep{zhang2021towards}. Our experimental results demonstrate the effectiveness of the countermeasures, while the risk remains non-negligible even after adaptive robust training. 
Therefore, our investigation suggests that practitioners have to be aware of this type of threat and be careful about dataset security.

\subsection{Method Description}

PGD-AT is a popular adversarial training method that minimizes cross-entropy loss on adversarial examples:
\begin{equation}
\label{eq:pgdat}
\textstyle
    \cL_{\pgdat} = \underset{(\vx, y) \sim \cD}{\bbE} \left[ \underset{\| \vx' - \vx \| \leq \epsilon}{\max} \ce(\vp_{\vtheta}(\vx'), y) \right].
\end{equation}
Though the objective of PGD-AT is originally designed to defend against adversarial examples, we are interested in its robustness against hypocritical perturbations in this paper.

TRADES is another adversarial training variant, whose training objective is:
\begin{equation}
\label{eq:trades}
\textstyle
    \cL_{\trades} = \underset{(\vx, y) \sim \cD}{\bbE}  \left[ \ce(\vp_{\vtheta}(\vx), y) + \lambda \cdot \kl(\vp_{\vtheta}(\vx) \parallel \vp_{\vtheta}(\vx_{\sta}) \right],
\end{equation}
where $\vx_{\sta}=\arg\max_{\| \vx' - \vx \| \leq \epsilon} \kl(\vp_{\vtheta}(\vx) \parallel \vp_{\vtheta}(\vx'))$, $\kl(\cdot\parallel\cdot)$ denotes the Kullback–Leibler divergence, and $\lambda$ is the hyperparameter to control the trade-off. We note that TRADES essentially aims to minimize a trade-off between natural risk and stability risk. Thus, it is reasonable to expect that TRADES performs better than PGD-AT for resisting hypocritical perturbations, as supported by Theorem~\ref{thrm:sta-risk-decomposition}.

We further consider an adaptive robust training objective:
\begin{equation}
\label{eq:thrm}
\textstyle
    \cL_{\thrm} = \underset{(\vx, y) \sim \cD}{\bbE}  \left[ \ce(\vp_{\vtheta}(\vx), y) + \lambda \cdot \kl(\vp_{\vtheta}(\vx) || \vp_{\vtheta}(\vx_{\hyp})) \right],
\end{equation}
where $\vx_{\hyp}=\arg\min_{\| \vx' - \vx \| \leq \epsilon} \ce(\vp_{\vtheta}(\vx'), y)$ as in Section~\ref{sec:hyp-examples}, and $\lambda$ is the hyperparameter to control the trade-off. We note that Eq.~(\ref{eq:thrm}) essentially aims to minimize a trade-off between natural risk and hypocritical risk (more details are provided in Appendix~\ref{app:thrm-details}). Thus, we term this method THRM, i.e., Trade-off for Hypocritical Risk Minimization.

Additionally, we adapt an inherently robust network architecture called $\linf$-dist nets~\citep{zhang2021towards} to resist hypocritical perturbations, whose technical details are deferred to Appendix~\ref{app:linfnet-details} due to the space limitation.

\begin{figure*}[t]
\vspace{-8px}
   \centering
   \subfigure[CIFAR-10]{
      \centering
      \includegraphics[width=0.30\textwidth]{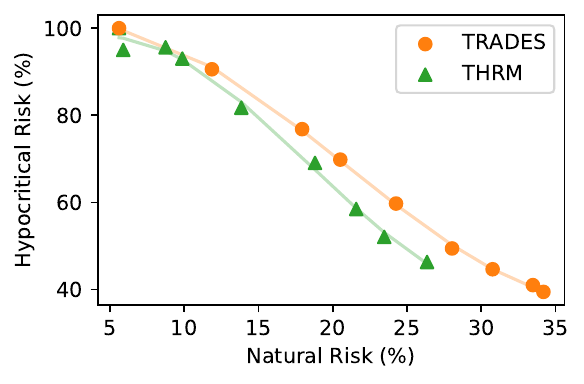}
   }
   \hfill
   \subfigure[CIFAR-100]{
      \centering
      \includegraphics[width=0.30\textwidth]{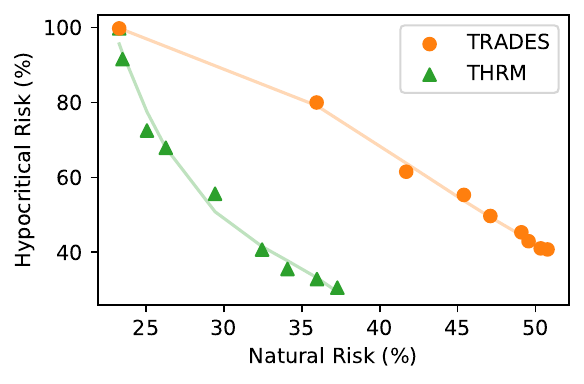}
   }
   \hfill
   \subfigure[Tiny-ImageNet]{
      \centering
      \includegraphics[width=0.30\textwidth]{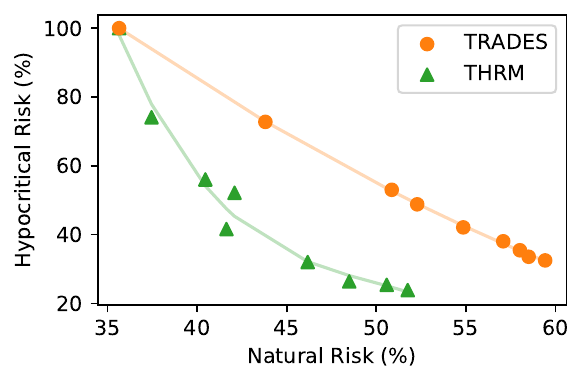}
   }
\vspace{-5px}
\caption{
Empirical comparison between TRADES and THRM in terms of natural risk and the hypocritical risk on misclassified examples under $\linf$ threat model.
Each point represents a model trained on the \textsl{Quality} data with a different $\lambda$.
}
\label{fig:tradeoff-3ds}
\vspace{-8px}
\end{figure*}

\subsection{Method Performance}
\label{sec:countermeasures-exps}

In this subsection, we evaluate the effectiveness of the countermeasures described above. From now on, we consider the \textsl{Poisoning} and \textsl{Quality} training sets for three reasons:
\textit{i)} the \textsl{Poisoning} data can be utilized to train accurate model via adversarial training~\citep{tao2021provable}.
\textit{ii)} adversarial training methods are hard to fit the \textsl{Noise} and \textsl{Mislabeling} training data~\citep{dong2021exploring};
\textit{iii)} the \textsl{Noise} and \textsl{Mislabeling} training data can be avoided by standard data cleaning~\citep{kandel2011wrangler}, while the \textsl{Poisoning} and \textsl{Quality} data cannot, since they are correctly labelled.

\subsubsection{Performance of PGD-AT and TRADES.}
Table~\ref{tab:at-resnet18-xxx} reports the results of PGD-AT and TRADES on CIFAR-10, CIFAR-100 and Tiny-ImageNet. We observe that the robustness of the models against hypocritical perturbations is better than the naturally trained (NT) models in Section~\ref{sec:hyp-examples-exps}, so is their robustness against adversarial perturbations. Nevertheless, there are still a large amount of misclassified examples that can be perturbed to be correctly classified. For example, the \textsl{Quality} (PGD-AT) model on CIFAR-10 exhibit 96.92\% accuracy on hypocritically perturbed examples, while its clean accuracy is only 84.08\%. Results on SVHN are deferred to Appendix~\ref{app:omit-figure} Table~\ref{tab:at-resnet18-svhn-full}, and similar conclusions hold.

\subsubsection{A Closer Look at Robustness.}
To directly compare the model robustness, we report the attack success rate of hypocritical attacks (which is equivalent to the hypocritical risk on misclassified examples) for the \textsl{Quality} models in Figure~\ref{fig:asr-quality}. As a reference, we also report the success rate of targeted adversarial examples on the misclassified examples. An interesting observation is that the attack success rate of hypocritical examples is much greater than that of the targeted adversarial examples, especially on CIFAR-100 and Tiny-ImageNet, indicating that hypocritical risk may be higher than we thought. More importantly, we observe that the attack success rate of TRADES is lower than that of PGD-AT on CIFAR-10, SVHN and CIFAR-100. This indicates that TRADES is not only better than PGD-AT for adversarial robustness (which is observed in~\citet{pang2020bag}) but also better than PGD-AT for hypocritical robustness. One exception is that TRADES performs worse on Tiny-ImageNet. This is simply because that we set the trade-off parameter of TRADES to 6 as in~\citep{zhang2019theoretically, pang2020bag}, which is too small for Tiny-ImageNet. In the next paragraph, this parameter will be tuned.

\subsubsection{Comparison with THRM.}
We empirically compare TRADES and THRM in terms of the natural risk and the hypocritical risk on misclassified examples by tuning the regularization parameter $\lambda$ in the range $[0, 100]$. The reported natural risk is estimated on clean verification data. The reported hypocritical risk is estimated on misclassified examples and is empirically approximated using PGD. Results for the models trained on the \textsl{Quality} data are summarized in Figure~\ref{fig:tradeoff-3ds}. Numerical details about the model accuracy on $\cD$, $\cA$, and $\cF$ with different $\lambda$ are given in Appendix~\ref{app:omit-figure} Tables~\ref{tab:tradeoff-cifar10},~\ref{tab:tradeoff-cifar100},~\ref{tab:tradeoff-tinyimagenet}, and~\ref{tab:tradeoff-cifar10-poisoning}.
We observe that for both TRADES and THRM, as $\lambda$ increases, the natural risk increases and the hypocritical risk decreases. It turns out that THRM achieves a better trade-off than TRADES in all cases, which is consistent with our analysis of THRM in Appendix~\ref{app:thrm-details}, and the gap between THRM and TRADES tends to increase when the number of classes is large. Therefore, when we only consider the threat of hypocritical attacks, THRM would be preferable than TRADES. However, if one wants to resist the threat of both adversarial examples and hypocritical examples, TRADES is a viable alternative.

\subsubsection{Results of $\linf$-dist nets.}
Results show that $\linf$-dist nets achieve moderate certified hypocritical risk. For both \textsl{Quality} model and \textsl{Poisoning} model, nearly half of the errors are guaranteed not to be covered up by any attack. However, $\linf$-dist nets still perform worse than TRADES and THRM in terms of empirical hypocritical risk.

Overall, some improvements have been obtained, while complete robustness against hypocritical attacks still cannot be fully achieved with the current methods. Hypocritical risk remains non-negligible even after adaptive robust training. This dilemma highlights the difficulty of stabilizing models to prevent hypocritical attacks. We feel that new manners may be needed to better tackle this problem.

\section{Conclusions and Future Directions}

This paper unveils the threat of hypocritical examples in the model verification stage. Our experimental results indicate that this type of security risk is pervasive, and remains non-negligible even if adaptive countermeasures are adopted. Therefore, our investigation suggests that practitioners should be aware of this type of threat and be careful about dataset security. Below we discuss some limitations with our current study, and we also feel that our results can lead to several thought-provoking future works.

\textit{Other performance requirements. }
One may consider using adversarial perturbations to combat hypocritical attacks, i.e., estimating the robust error~\citep{zhang2019theoretically} on the verification data. We note that this is actually equivalent to choosing the robust error as the performance requirement. It is natural then to ask whether a hypocritical attacker can cause a substandard model to exhibit high robust accuracy with small perturbations. We leave this as future work.

\textit{Transferability. }
It is also very important to study the transferability of hypocritical examples across substandard models models. Transfer-based hypocritical attacks are still harmful when model structure and weights are unknown to the attacker. Understanding the transferability would help us to design effective defense strategies against the transfer-based hypocritical attacks.

\textit{Good use of hypocritical perturbations. }
We showed that many types of substandard models are susceptible to hypocritical attacks. Then, an intriguing question is whether we can turn this weakness into a strength. Specifically, one may find such a ``true friend'' who is capable of \textit{consistently} helping a substandard model during the deployment stage to make correct predictions. There are concurrent works~\citep{salman2021unadversarial, pestana2021assistive} which explored this direction, where ``robust objects'' are designed to help a model to confidently detect or classify them.

\clearpage

\section*{Acknowledgments}


This work was supported by the National Natural Science Foundation of China (Grant No. 61732006, 62076124, and 62106028).
Lei Feng was also supported by CAAI-Huawei MindSpore Open Fund.

\small
\bibliography{aaai22}

\clearpage
\onecolumn 
\normalsize
\appendix

\section{Experimental Setup}
\label{app:exp-setup}

To reveal the worst-case risk, our experiments focus on white-box attacks. Therefore, the hypocritical examples in Section~\ref{sec:hyp-examples} and Section~\ref{sec:countermeasures} are carefully crafted for each model after training. We also note that it is possible for future research to compute hypocritical examples before training via transfer-based attacks. Our experiments run with NVIDIA GeForce RTX 2080 Ti GPUs. Our implementation is based on PyTorch, and the code to reproduce our results available at \url{https://github.com/TLMichael/False-Friends}.

\subsection{Datasets}

\subsubsection{CIFAR-10\footnote{\url{https://www.cs.toronto.edu/~kriz/cifar.html}}.}
This dataset~\citep{krizhevsky2009learning} consists of 60,000 $32\times 32$ colour images (50,000 images for training and 10,000 images for evaluation) in 10 classes (``airplane'', ``car'', ``bird'', ``cat'', ``deer'', ``dog'', ``frog'', ``horse'', ``ship'', and ``truck''). We use the original test set as the clean verification data. We adopt various architectures for this dataset, including MLP, VGG-16~\citep{simonyan2014very}, ResNet-18~\citep{he2016deep}, and WideResNet-28-10~\citep{zagoruyko2016wide}. A four-layer MLP (2 hidden layers, 3072 neurons in each) with ReLU activations~\citep{glorot2011deep} is used here. The models are initialized with the commonly used He initialization~\citep{he2015delving}, which is the default initialization in torchvision\footnote{\url{https://github.com/pytorch/vision/tree/master/torchvision/models}}. The initial learning rate is set to 0.1, except for MLP and VGG-16 on Noise and Mislabeling, where the initial learning rate is set to 0.01. Following~\citet{pang2020bag}, in the default setting, the models are trained for 110 epochs using SGD with momentum 0.9, batch size 128, and weight decay $5\times 10^{-4}$, where the learning rate is decayed by a factor of 0.1 in the 100th and 105th epochs, and simple data augmentations such as $32 \times 32$ random crop with 4-pixel padding and random horizontal flip are applied.  Following~\citet{zhang2016understanding}, when training on the \textsl{Noise} and \textsl{Mislabeling} data, data augmentations and weight decay are turned off\footnote{\url{https://github.com/pluskid/fitting-random-labels}}.

\subsubsection{SVHN\footnote{\url{http://ufldl.stanford.edu/housenumbers/}}.}
This dataset~\citep{netzer2011reading} consists of 630,420 $32\times 32$ colour images (73,257 images for training and 26,032 images for evaluation) in 10 classes. We use the original test set as the clean verification data. We adopt the ResNet-18 architecture for this dataset. The model is initialized with the commonly used He initialization. The initial learning rate is set to 0.01. In the default setting, the models are trained for 60 epochs using SGD with momentum 0.9, batch size 128, and weight decay $5\times 10^{-4}$, where the learning rate is decayed by a factor of 0.1 in the 50th and 55th epochs, and simple data augmentation such as $32 \times 32$ random crop with 4-pixel padding is applied. when training on the \textsl{Noise} and \textsl{Mislabeling} data, data augmentation and weight decay are turned off.

\subsubsection{CIFAR-100\footnote{\url{https://www.cs.toronto.edu/~kriz/cifar.html}}.}
This dataset~\citep{krizhevsky2009learning} consists of 60,000 $32\times 32$ colour images (50,000 images for training and 10,000 images for evaluation) in 100 classes. We use the original test set as the clean verification data. We adopt the ResNet-18 architecture for this dataset. The model is initialized with the commonly used He initialization. The initial learning rate is set to 0.1. In the default setting, the models are trained for 110 epochs using SGD with momentum 0.9, batch size 128, and weight decay $5\times 10^{-4}$, where the learning rate is decayed by a factor of 0.1 in the 100th and 105th epochs, and simple data augmentations such as $32 \times 32$ random crop with 4-pixel padding and random horizontal flip are applied. when training on the \textsl{Noise} and \textsl{Mislabeling} data, data augmentations and weight decay are turned off.

\subsubsection{Tiny-ImageNet\footnote{\url{http://cs231n.stanford.edu/tiny-imagenet-200.zip}}.}
This dataset~\citep{yao2015tiny} consists of 110000 $64\times 64$ colour images (100,000 images for training and 10,000 images for evaluation) in 200 classes. We use the original validation set as the clean verification data. We adopt the ResNet-18 architecture for this dataset. The model is initialized with the commonly used He initialization. The initial learning rate is set to 0.1. In the default setting, the models are trained for 60 epochs using SGD with momentum 0.9, batch size 64, and weight decay $5\times 10^{-4}$, where the learning rate is decayed by a factor of 0.1 in the 50th and 55th epochs, and simple data augmentations such as $64 \times 64$ random crop with 4-pixel padding and random horizontal flip are applied. when training on the \textsl{Noise} and \textsl{Mislabeling} data, data augmentation and weight decay are turned off.

\subsection{Robust Training}

We perform robust training algorithms including PGD-AT, TRADES, and THRM by following the common settings~\citep{madry2018towards, pang2020bag}. Specifically, we train against a projected gradient descent (PGD) attacker, starting from a random initial perturbation of the training data. 
We use 10/20 steps of PGD with a step size of $\epsilon/4$ for training/evaluation.
We consider adversarial perturbations in $\ell_p$ norm where $p=\{2, \infty\}$. For $\ell_2$ threat model, $\epsilon = 0.5$. For $\linf$ threat model, $\epsilon=8/255$.

\subsection{Low-Quality Training Data}

The \textsl{Noisy} data is constructed by replacing the each image in the training set with uniform noise and keeping the labels unchanged~\citep{zhang2016understanding}. The \textsl{Mislabeling} data is constructed by replacing the labels with random ones and keeping the labels unchanged~\citep{zhang2016understanding}. The \textsl{Poisoning} data is constructed by perturbing the images to maximize generalization error and keeping the labels unchanged~\citep{ilyas2019adversarial, nakkiran2019a}. Following~\citet{tao2021provable}, each input-label pair $(\vx, y)$ in the \textsl{Poisoning} data is constructed as follows. We first select a target class $t$ deterministically according to the source class $y$ (e.g., using a fixed permutation of labels). Then, we add a small adversarial perturbation to $\vx$ to ensure that it is misclassified as $t$ by as naturally trained model, that is, $\vx_{\adv} = \arg\min_{\| \vx' - \vx \| \leq \epsilon} \operatorname{CE}(f(\vx), t)$. Here $f$ is a naturally trained model on the \textsl{Quality} data. Finally, we assign the correct label $y$ to the perturbed input. The resulting input-label pairs $(\vx_{\adv}, y)$ make up the \textsl{Poisoning} dataset. Perturbations are constrained in $\ell_2$-norm with $\epsilon=0.5$ or $\linf$-norm with $\epsilon=8/255$. The number of PGD iterations is set to 100 and step size is set to $\epsilon/5$.

\section{Omitted Tables and Figures}
\label{app:omit-figure}

\newcommand{\hexp}[1]{\includegraphics[width=0.18\textwidth]{#1}}
\newcolumntype{C}{>{\centering\arraybackslash}m{0.18\textwidth}}

\begin{table}[!h]
    \centering
    \caption{
    Verification examples for the naturally trained models. These examples are originally misclassified (as red labels) by the models, but they are correctly classified (as green labels) after adding hypocritical perturbations. Perturbations are rescaled for display.
    }
    \label{tab:my_label}
    \begin{tabular}{cCCCC}
        \toprule
        Model & Example \#422 & Example \#4503 & Example \#8778 & Example \#9910 \\
        \midrule
        Naive       & \hexp{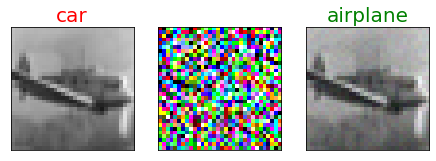}& \hexp{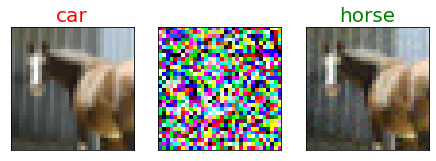}& \hexp{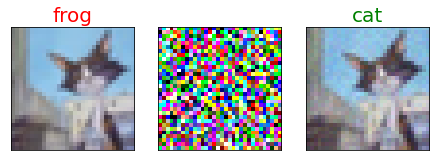}& \hexp{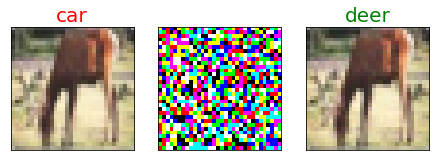}  \\
        Noise       & \hexp{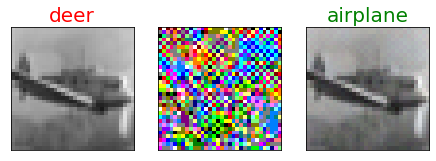}& \hexp{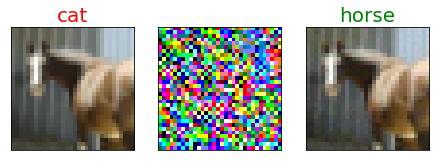}& \hexp{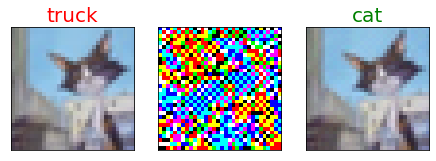}& \hexp{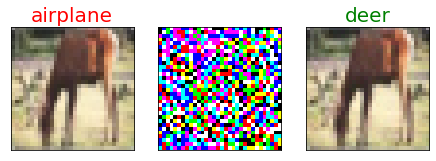}  \\
        Mislabeling & \hexp{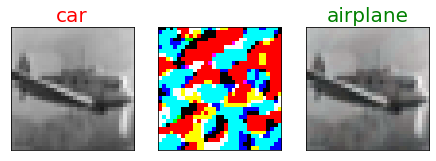}& \hexp{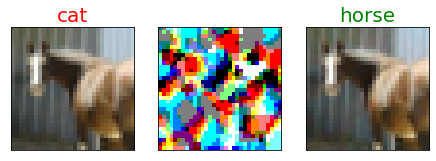}& \hexp{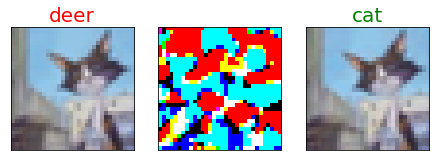}& \hexp{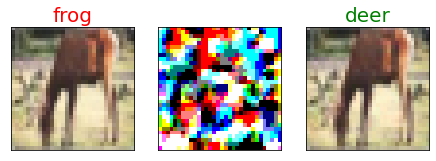}  \\
        Poisoning   & \hexp{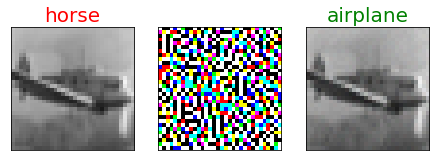}& \hexp{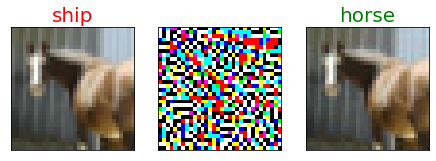}& \hexp{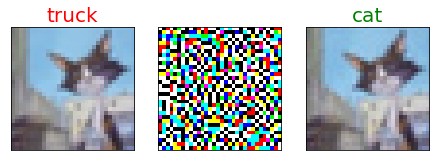}& \hexp{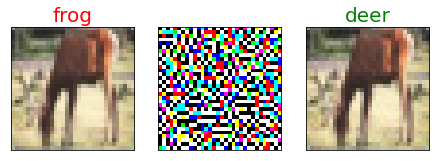}  \\
        Quality     & \hexp{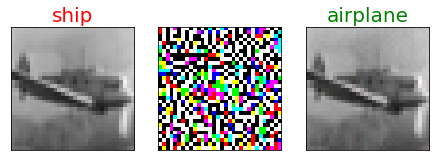}& \hexp{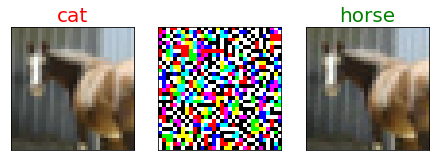}& \hexp{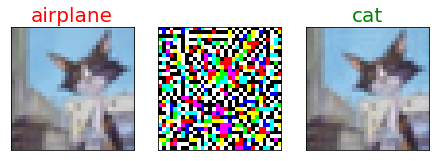}& \hexp{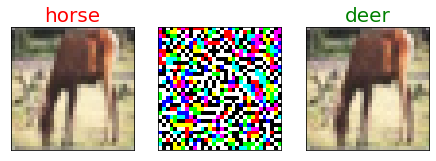}  \\
        \bottomrule
    \end{tabular}
\end{table}

\begin{table}[!h]
\tiny
\centering
\caption{
Verification accuracy (\%) of substandard models on CIFAR-10 under $\linf$ threat model across different architectures.
$\cD$, $\cA$, and $\cF$ denote the model accuracy evaluated on clean examples, adversarially perturbed examples, and hypocritically perturbed examples, respectively.
We report mean and standard deviation over 5 random runs.
}
\label{tab:st-xxx-cifar10-full}
\begin{tabular}{@{}l|ccc|ccc|ccc@{}}
\toprule
\multirow{2}{*}{Model} & \multicolumn{3}{c|}{MLP} & \multicolumn{3}{c|}{VGG-16} & \multicolumn{3}{c}{WideResNet-28-10} \\ \cmidrule(lr){2-4} \cmidrule(lr){5-7} \cmidrule(lr){8-10}
            & $\cD$ & $\cA$ & $\cF$ & $\cD$ & $\cA$ & $\cF$ & $\cD$ & $\cA$ & $\cF$ \\ \midrule
Naive                  & 8.56±1.20 & 0.00±0.00 & 99.56±0.50 & 9.76±0.34 & 0.45±0.99 & 57.25±26.53 & 10.06±0.14 & 0.34±0.65 & 40.64±14.66 \\
Noise                  & 8.53±0.77 & 0.71±0.46 & 86.75±0.34 & 9.81±1.17 & 0.00±0.01 & 97.98±3.60  & 11.35±1.22 & 0.02±0.05 & 98.42±1.66  \\
Mislabeling & 9.92±0.18  & 0.00±0.00 & 100.00±0.00 & 9.94±0.20  & 0.00±0.00 & 99.86±0.10  & 10.21±0.17 & 0.00±0.00 & 100.00±0.00 \\
Poisoning   & 57.60±0.17 & 0.55±0.05 & 99.50±0.05  & 12.19±0.88 & 0.00±0.00 & 99.49±0.17  & 10.42±1.16 & 0.00±0.00 & 100.00±0.00 \\
Quality     & 58.09±0.21 & 0.91±0.07 & 99.31±0.07  & 92.90±0.08  & 0.00±0.00 & 100.00±0.00 & 95.41±0.13 & 0.00±0.00 & 100.00±0.00 \\ \bottomrule
\end{tabular}
\end{table}

\begin{table}[!h]
\tiny
\centering
\caption{
Verification accuracy (\%) of substandard models on CIFAR-10 under $\ell_2$ threat model across different architectures.
$\cD$, $\cA$, and $\cF$ denote the model accuracy evaluated on clean examples, adversarially perturbed examples, and hypocritically perturbed examples, respectively.
We report mean and standard deviation over 5 random runs.
}
\label{tab:st-xxx-cifar10-l2-full}
\begin{tabular}{@{}l|ccc|ccc|ccc@{}}
\toprule
\multirow{2}{*}{Model} & \multicolumn{3}{c|}{MLP} & \multicolumn{3}{c|}{VGG-16} & \multicolumn{3}{c}{WideResNet-28-10} \\ \cmidrule(lr){2-4} \cmidrule(lr){5-7} \cmidrule(lr){8-10}
            & $\cD$ & $\cA$ & $\cF$  & $\cD$ & $\cA$ & $\cF$  & $\cD$ & $\cA$ & $\cF$  \\ \midrule
Naive                  & 8.56±1.20  & 0.07±0.05  & 72.75±2.79 & 9.76±0.34  & 3.28±4.22 & 33.24±23.10 & 10.06±0.14 & 4.85±3.95 & 19.52±8.43  \\
Noise                  & 8.53±0.77  & 0.79±0.45  & 74.32±2.44 & 11.13±1.32 & 0.00±0.00 & 76.63±12.49 & 9.07±1.56  & 0.34±0.36 & 81.83±10.68 \\
Mislabeling            & 9.92±0.18  & 0.00±0.00  & 97.75±0.08 & 10.00±0.46 & 0.01±0.01 & 89.71±1.20  & 10.07±0.14 & 0.00±0.00 & 99.99±0.01  \\
Poisoning & 57.97±0.28 & 19.70±0.43 & 88.52±0.25 & 23.61±1.45 & 0.00±0.00 & 99.85±0.13 & 18.57±1.35 & 0.00±0.00 & 100.00±0.00 \\
Quality                & 58.09±0.21 & 19.76±0.51 & 88.74±0.20 & 92.99±0.14 & 0.21±0.06 & 100.00±0.00 & 95.44±0.09 & 0.03±0.01 & 100.00±0.00 \\ \bottomrule
\end{tabular}
\end{table}

\begin{table}[!h]
\tiny
\centering
\caption{
Verification accuracy (\%) of substandard ResNet-18 models under $\linf$ threat model across different datasets. 
$\cD$, $\cA$, and $\cF$ denote the model accuracy evaluated on clean examples, adversarially perturbed examples, and hypocritically perturbed examples, respectively.
We report mean and standard deviation over 5 random runs.
}
\label{tab:st-resnet18-xxx-full}
\begin{tabular}{@{}l|ccc|ccc|ccc@{}}
\toprule
\multirow{2}{*}{Model} & \multicolumn{3}{c|}{SVHN} & \multicolumn{3}{c|}{CIFAR-100} & \multicolumn{3}{c}{Tiny-ImageNet} \\ \cmidrule(lr){2-4} \cmidrule(lr){5-7} \cmidrule(lr){8-10}
            & $\cD$ & $\cA$ & $\cF$  & $\cD$ & $\cA$ & $\cF$  & $\cD$ & $\cA$ & $\cF$  \\ \midrule
Naive                  & 10.73±4.30 & 0.85±1.89 & 55.64±23.16 & 0.98±0.09 & 0.02±0.05 & 7.26±3.12 & 0.49±0.05 & 0.04±0.08 & 4.50±3.17 \\
Noise       & 10.76±2.04 & 0.00±0.00 & 99.96±0.09  & 1.02±0.14  & 0.01±0.01 & 81.93±23.14 & 0.39±0.05  & 0.01±0.01 & 74.58±27.88 \\
Mislabeling & 9.77±0.16  & 0.00±0.00 & 100.00±0.00 & 0.99±0.08  & 0.00±0.00 & 99.81±0.10  & 0.48±0.08  & 0.00±0.00 & 99.99±0.01  \\
Poisoning   & 41.42±0.73 & 0.00±0.00 & 100.00±0.00 & 34.80±0.56 & 0.00±0.00 & 100.00±0.00 & 34.87±0.27 & 0.00±0.00 & 100.00±0.00 \\
Quality     & 96.57±0.14 & 0.32±0.05 & 99.99±0.00  & 76.64±0.18 & 0.01±0.00 & 99.96±0.03  & 64.03±0.14 & 0.02±0.02 & 100.00±0.00 \\ \bottomrule
\end{tabular}
\end{table}

\begin{table}[!h]
\tiny
\centering
\caption{
Verification accuracy (\%) of adversarially trained ResNet-18 models under $\linf$ threat model across different datasets. 
$\cD$, $\cA$, and $\cF$ denote the model accuracy evaluated on clean examples, adversarially perturbed examples, and hypocritically perturbed examples, respectively.
We report mean and standard deviation over 5 random runs.
}
\label{tab:at-resnet18-xxx-full}
\begin{tabular}{@{}l|ccc|ccc|ccc@{}}
\toprule
\multirow{2}{*}{Model} & \multicolumn{3}{c|}{CIFAR-10} & \multicolumn{3}{c|}{CIFAR-100} & \multicolumn{3}{c}{Tiny-ImageNet} \\ \cmidrule(lr){2-4} \cmidrule(lr){5-7} \cmidrule(lr){8-10}
          & $\cD$ & $\cA$ & $\cF$ & $\cD$ & $\cA$ & $\cF$ & $\cD$ & $\cA$ & $\cF$ \\ \midrule
Poisoning (PGD-AT)         & 82.65±0.14 & 51.34±0.23 & 96.2±0.07  & 57.32±0.16 & 27.85±0.13 & 83.2±0.14  & 45.14±0.25 & 21.69±0.11 & 68.96±0.29 \\
Poisoning (TRADES) & 80.01±0.17 & 52.34±0.18 & 94.64±0.08 & 56.29±0.33 & 29.41±0.25 & 81.79±0.17 & 46.38±0.20 & 21.41±0.10 & 73.5±0.21 \\
Quality (PGD-AT)           & 84.08±0.18 & 51.98±0.24 & 96.92±0.09 & 59.19±0.27 & 28.21±0.16 & 84.87±0.10 & 47.23±0.22 & 22.03±0.17 & 71.95±0.22 \\
Quality (TRADES)           & 81.05±0.14 & 53.32±0.21 & 95.17±0.18 & 57.27±0.04 & 30.00±0.30 & 82.88±0.32 & 47.86±0.09 & 22.03±0.16 & 75.04±0.18 \\ \bottomrule
\end{tabular}
\end{table}

\begin{table}[!h]
\tiny
\centering
\caption{
Verification accuracy (\%) of adversarially trained ResNet-18 models on SVHN under $\linf$ threat model. 
$\cD$, $\cA$, and $\cF$ denote the model accuracy evaluated on clean examples, adversarially perturbed examples, and hypocritically perturbed examples, respectively.
We report mean and standard deviation over 5 random runs.
}
\label{tab:at-resnet18-svhn-full}
\begin{tabular}{@{}l|ccc@{}}
\toprule
\multirow{2}{*}{Model} & \multicolumn{3}{c}{SVHN}             \\ \cmidrule(lr){2-4} 
                       & $\cD$      & $\cA$      & $\cF$      \\ \midrule
Poisoning (PGD-AT)     & 90.50±0.14 & 54.17±0.16 & 99.28±0.02 \\
Poisoning (TRADES)     & 87.34±0.40 & 54.92±0.15 & 98.38±0.10 \\
Quality (PGD-AT)       & 93.37±1.49 & 56.29±2.21 & 99.45±0.13 \\
Quality (TRADES)       & 90.69±0.30 & 59.78±0.31 & 98.90±0.07 \\ \bottomrule
\end{tabular}
\end{table}

\begin{figure}[!h]
   \centering
   \subfigure[CIFAR-10]{
      \centering
      \includegraphics[width=0.48\textwidth]{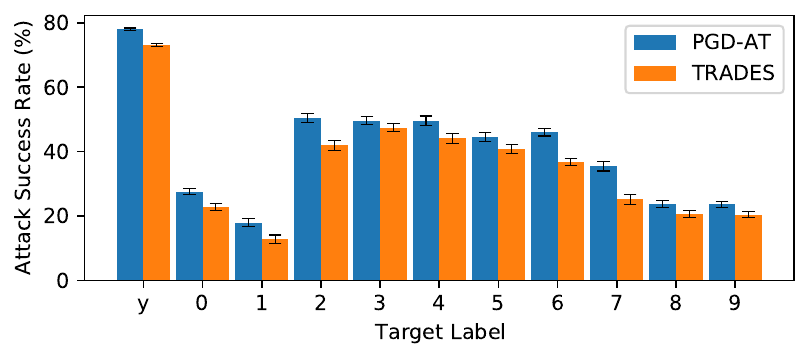}
   }
   \hfill
   \subfigure[SVHN]{
      \centering
      \includegraphics[width=0.48\textwidth]{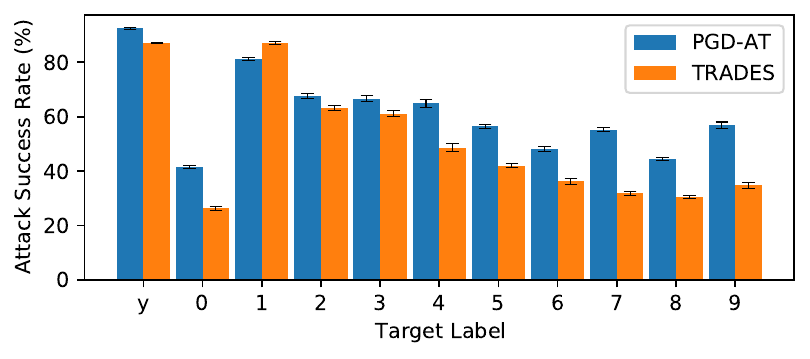}
   }
   \subfigure[CIFAR-100]{
      \centering
      \includegraphics[width=0.48\textwidth]{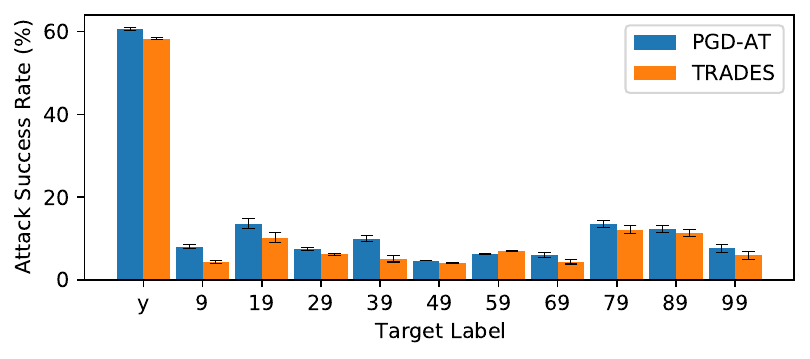}
   }
   \hfill
   \subfigure[Tiny-ImageNet]{
      \centering
      \includegraphics[width=0.48\textwidth]{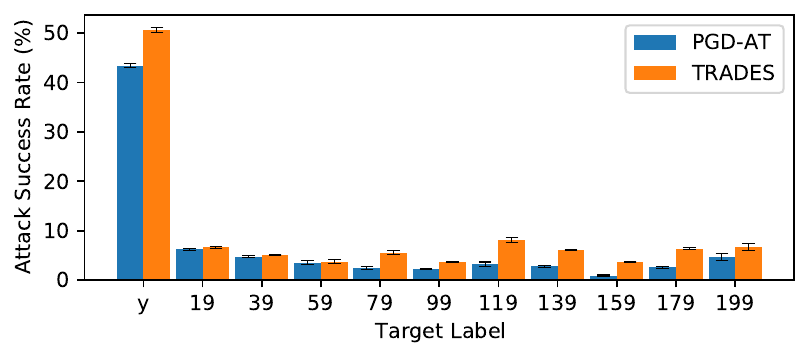}
   }
\caption{
Attack success rate (\%) of adversarially trained ResNet-18 models on misclassified examples under $\linf$ threat model. Here, the models are trained on the \textsl{Poisoning} data, and conclusions similar to Figure~\ref{fig:asr-quality} in the main text similarlly hold. The target label ``y'' denotes that the misclassified examples are perturbed to be correctly classified. The target labels ``0'' $\sim$ ``199'' denote that the misclassified examples are perturbed to be classified as a specific target, no matter whether the target label is correct or not. Error bars indicate standard deviation over 5 random runs. 
}
\label{fig:asr-poisoning}
\end{figure}

\begin{figure}[!h]
   \centering
    \includegraphics[width=0.3\textwidth]{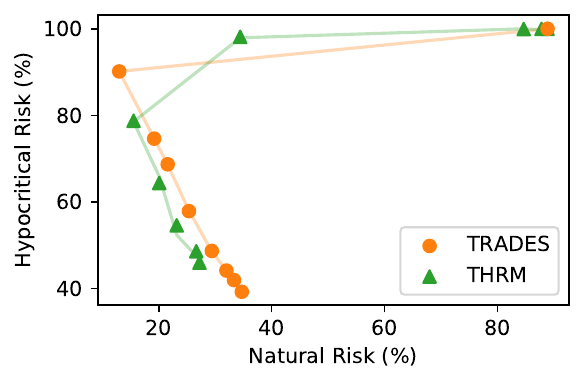}
\caption{
Empirical comparison between TRADES and THRM in terms of the natural risk and the hypocritical risk on misclassified examples under $\linf$ threat model.
Each point represents a model with a different $\lambda$.
Here, the models are trained on the \textsl{Poisoning} data of CIFAR-10.
We observe that, when $\lambda$ is small, TRADES is better than THRM.
This is because TRADES can defense against the Poisoning data~\citep{nakkiran2019a} by minimizing the adversarial risk~\citep{tao2021provable}, while THRM cannot.
However, we note that, when $\lambda$ is large, THRM can achieve a better trade-off than TRADES.
This may be due to that training model using THRM with large $\lambda$ can somewhat bring nontrivial robustness to adversarial perturbations, as shown in Table~\ref{tab:tradeoff-cifar10-poisoning} and Figure~\ref{fig:tradeoff-3ds-adv}.
}
\label{fig:tradeoff-cifar10-poisoning}
\end{figure}

\begin{figure}[h]
   \centering
   \subfigure[CIFAR-10]{
      \centering
      \includegraphics[width=0.30\textwidth]{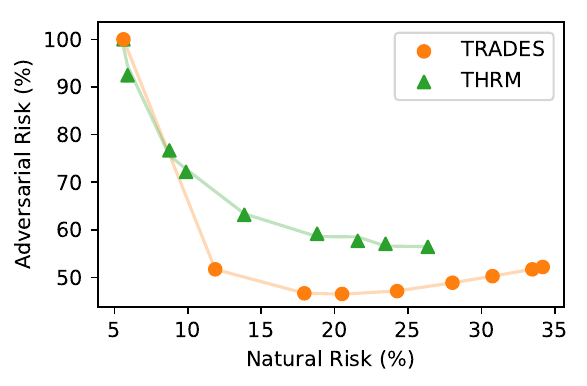}
   }
   \hfill
   \subfigure[CIFAR-100]{
      \centering
      \includegraphics[width=0.30\textwidth]{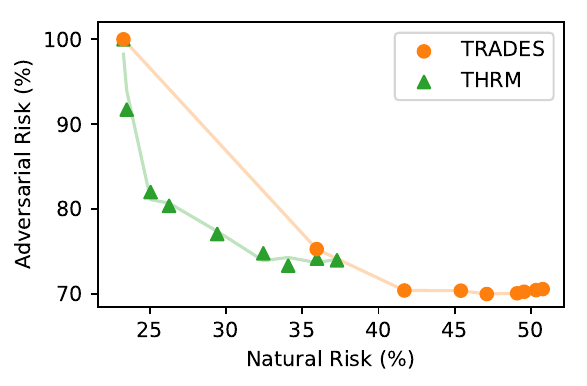}
   }
   \hfill
   \subfigure[Tiny-ImageNet]{
      \centering
      \includegraphics[width=0.30\textwidth]{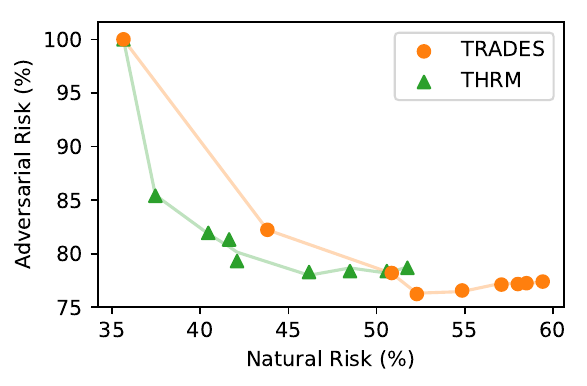}
   }
\caption{
Empirical comparison between TRADES and THRM in terms of natural risk and the adversarial risk on all examples under $\linf$ threat model.
Each point represents a model trained on the \textsl{Quality} data with a different $\lambda$.
\textbf{Description:} This is a plot to illustrate adversarial risk.
The results show that \textit{THRM achieves considerable adversarial robustness} compared with TRADES. We note that THRM is only designed to minimize natural and hypocritical risks. Thus, it is intriguing in some sense that it can achieve considerable adversarial risk. A possible explanation for this might be that improving the hypocritical risk somehow decreases the Lipschitz constant of the model, thereby reducing the adversarial risk.
}
\label{fig:tradeoff-3ds-adv}
\end{figure}

\begin{table}[!h]
\small
\centering
\caption{Empirical comparison between TRADES and THRM on CIFAR-10 under $\linf$ threat model. The models are trained on the \textsl{Quality} data.}
\label{tab:tradeoff-cifar10}
\begin{tabular}{@{}lccccccc@{}}
\toprule
\multirow{2}{*}{Method} & \multirow{2}{*}{$\lambda$} & \multicolumn{3}{c}{Accuracy} & \multicolumn{3}{c}{Risk}                                                   \\ \cmidrule(lr){3-5} \cmidrule(lr){6-8} 
                        &                            & $\cD$   & $\cA$   & $\cF$    & $\cR_{\nat}(f, \cD)$ & $\cR_{\adv}(f, \cD_f^+)$ & $\cR_{\hyp}(f, \cD_f^-)$ \\ \midrule
\multirow{9}{*}{THRM}   & 0                          & 94.38   & 0.00    & 100.00   & 5.62                 & 100.00                   & 100.00                   \\
                        & 1                          & 94.10   & 7.53    & 99.70    & 5.90                 & 92.00                    & 94.92                    \\
                        & 5                          & 91.25   & 23.32   & 99.65    & 8.75                 & 74.44                    & 96.00                    \\
                        & 10                         & 90.12   & 27.83   & 99.31    & 9.88                 & 69.12                    & 93.02                    \\
                        & 20                         & 86.15   & 36.82   & 97.47    & 13.85                & 57.26                    & 81.73                    \\
                        & 40                         & 81.19   & 40.80   & 94.18    & 18.81                & 49.75                    & 69.06                    \\
                        & 60                         & 78.42   & 42.30   & 91.04    & 21.58                & 46.06                    & 58.48                    \\
                        & 80                         & 76.53   & 42.91   & 88.75    & 23.47                & 43.93                    & 52.07                    \\
                        & 100                        & 73.65   & 43.56   & 85.85    & 26.35                & 40.86                    & 46.30                    \\ \midrule
\multirow{9}{*}{TRADES} & 0                          & 94.38   & 0.00    & 100.00   & 5.62                 & 100.00                   & 100.00                   \\
                        & 1                          & 88.13   & 48.26   & 98.88    & 11.87                & 45.24                    & 90.56                    \\
                        & 5                          & 82.06   & 53.30   & 95.84    & 17.94                & 35.05                    & 76.81                    \\
                        & 10                         & 79.49   & 53.44   & 93.81    & 20.51                & 32.77                    & 69.82                    \\
                        & 20                         & 75.74   & 52.84   & 90.23    & 24.26                & 30.24                    & 59.73                    \\
                        & 40                         & 71.97   & 51.08   & 85.82    & 28.03                & 29.03                    & 49.41                    \\
                        & 60                         & 69.24   & 49.70   & 82.97    & 30.76                & 28.22                    & 44.64                    \\
                        & 80                         & 66.53   & 48.27   & 80.25    & 33.47                & 27.45                    & 40.99                    \\
                        & 100                        & 65.82   & 47.78   & 79.30    & 34.18                & 27.41                    & 39.44                    \\ \bottomrule
\end{tabular}
\end{table}

\begin{table}[!h]
\small
\centering
\caption{Empirical comparison between TRADES and THRM on CIFAR-100 under $\linf$ threat model. The models are trained on the \textsl{Quality} data.}
\label{tab:tradeoff-cifar100}
\begin{tabular}{@{}lccccccc@{}}
\toprule
\multirow{2}{*}{Method} & \multirow{2}{*}{$\lambda$} & \multicolumn{3}{c}{Accuracy} & \multicolumn{3}{c}{Risk}                                                   \\ \cmidrule(lr){3-5} \cmidrule(lr){6-8} 
                        &                            & $\cD$   & $\cA$   & $\cF$    & $\cR_{\nat}(f, \cD)$ & $\cR_{\adv}(f, \cD_f^+)$ & $\cR_{\hyp}(f, \cD_f^-)$ \\ \midrule
\multirow{9}{*}{THRM}   & 0                          & 76.72    & 0.01    & 99.95   & 23.28                & 99.99                    & 99.79                    \\
                        & 1                          & 76.51    & 8.28    & 98.02   & 23.49                & 89.18                    & 91.57                    \\
                        & 5                          & 74.94    & 18.00   & 93.11   & 25.06                & 75.98                    & 72.51                    \\
                        & 10                         & 73.73    & 19.62   & 91.57   & 26.27                & 73.39                    & 67.91                    \\
                        & 20                         & 70.57    & 22.94   & 86.94   & 29.43                & 67.49                    & 55.62                    \\
                        & 40                         & 67.55    & 25.21   & 80.76   & 32.45                & 62.68                    & 40.71                    \\
                        & 60                         & 65.92    & 26.67   & 78.03   & 34.08                & 59.54                    & 35.53                    \\
                        & 80                         & 64.03    & 25.84   & 75.83   & 35.97                & 59.64                    & 32.81                    \\
                        & 100                        & 62.72    & 26.03   & 74.10   & 37.28                & 58.50                    & 30.53                    \\ \midrule
\multirow{9}{*}{TRADES} & 0                          & 76.72    & 0.01    & 99.95   & 23.28                & 99.99                    & 99.79                    \\
                        & 1                          & 64.05    & 24.72   & 92.81   & 35.95                & 61.41                    & 80.00                    \\
                        & 5                          & 58.30    & 29.59   & 83.95   & 41.70                & 49.25                    & 61.51                    \\
                        & 10                         & 54.60    & 29.61   & 79.71   & 45.40                & 45.77                    & 55.31                    \\
                        & 20                         & 52.90    & 30.01   & 76.30   & 47.10                & 43.27                    & 49.68                    \\
                        & 40                         & 50.91    & 29.92   & 73.16   & 49.09                & 41.23                    & 45.32                    \\
                        & 60                         & 50.45    & 29.74   & 71.73   & 49.55                & 41.05                    & 42.95                    \\
                        & 80                         & 49.67    & 29.55   & 70.32   & 50.33                & 40.51                    & 41.03                    \\
                        & 100                        & 49.23    & 29.43   & 69.92   & 50.77                & 40.22                    & 40.75                    \\ \bottomrule
\end{tabular}
\end{table}

\begin{table}[!h]
\small
\centering
\caption{Empirical comparison between TRADES and THRM on Tiny-ImageNet under $\linf$ threat model. The models are trained on the \textsl{Quality} data.}
\label{tab:tradeoff-tinyimagenet}
\begin{tabular}{@{}lccccccc@{}}
\toprule
\multirow{2}{*}{Method} & \multirow{2}{*}{$\lambda$} & \multicolumn{3}{c}{Accuracy} & \multicolumn{3}{c}{Risk}                                                   \\ \cmidrule(lr){3-5} \cmidrule(lr){6-8} 
                        &                            & $\cD$   & $\cA$   & $\cF$    & $\cR_{\nat}(f, \cD)$ & $\cR_{\adv}(f, \cD_f^+)$ & $\cR_{\hyp}(f, \cD_f^-)$ \\ \midrule
\multirow{9}{*}{THRM}   & 0                          & 64.35   & 0.00    & 100.00   & 35.65                & 100.00                   & 100.00                   \\
                        & 1                          & 62.54   & 14.59   & 90.30    & 37.46                & 76.67                    & 74.11                    \\
                        & 5                          & 57.91   & 18.06   & 79.87    & 42.09                & 68.81                    & 52.17                    \\
                        & 10                         & 59.54   & 18.67   & 82.21    & 40.46                & 68.64                    & 56.03                    \\
                        & 20                         & 58.36   & 20.69   & 75.69    & 41.64                & 64.55                    & 41.62                    \\
                        & 40                         & 53.83   & 21.71   & 68.61    & 46.17                & 59.67                    & 32.01                    \\
                        & 60                         & 51.51   & 21.63   & 64.31    & 48.49                & 58.01                    & 26.40                    \\
                        & 80                         & 49.42   & 21.62   & 62.27    & 50.58                & 56.25                    & 25.41                    \\
                        & 100                        & 48.25   & 21.31   & 60.60    & 51.75                & 55.83                    & 23.86                    \\ \midrule
\multirow{9}{*}{TRADES} & 0                          & 64.35   & 0.00    & 100.00   & 35.65                & 100.00                   & 100.00                   \\
                        & 1                          & 56.19   & 17.77   & 88.07    & 43.81                & 68.38                    & 72.77                    \\
                        & 5                          & 49.14   & 21.80   & 76.11    & 50.86                & 55.64                    & 53.03                    \\
                        & 10                         & 47.72   & 23.76   & 73.24    & 52.28                & 50.21                    & 48.81                    \\
                        & 20                         & 45.15   & 23.43   & 68.25    & 54.85                & 48.11                    & 42.11                    \\
                        & 40                         & 42.92   & 22.88   & 64.65    & 57.08                & 46.69                    & 38.07                    \\
                        & 60                         & 41.99   & 22.83   & 62.56    & 58.01                & 45.63                    & 35.46                    \\
                        & 80                         & 41.50   & 22.75   & 61.14    & 58.50                & 45.18                    & 33.57                    \\
                        & 100                        & 40.58   & 22.60   & 59.90    & 59.42                & 44.31                    & 32.51                    \\ \bottomrule
\end{tabular}
\end{table}

\begin{table}[!h]
\small
\centering
\caption{Empirical comparison between TRADES and THRM on CIFAR-10 under $\linf$ threat model. The models are trained on the \textsl{Poisoning} data.}
\label{tab:tradeoff-cifar10-poisoning}
\begin{tabular}{@{}lccccccc@{}}
\toprule
\multirow{2}{*}{Method} & \multirow{2}{*}{$\lambda$} & \multicolumn{3}{c}{Accuracy} & \multicolumn{3}{c}{Risk}                                                   \\ \cmidrule(lr){3-5} \cmidrule(lr){6-8} 
                        &                            & $\cD$   & $\cA$   & $\cF$    & $\cR_{\nat}(f, \cD)$ & $\cR_{\adv}(f, \cD_f^+)$ & $\cR_{\hyp}(f, \cD_f^-)$ \\ \midrule
\multirow{9}{*}{THRM}   & 0                          & 11.13   & 0.00    & 100.00   & 88.87                & 100.00                   & 100.00                   \\
                        & 1                          & 12.20   & 0.00    & 99.97    & 87.80                & 100.00                   & 99.97                    \\
                        & 5                          & 15.35   & 0.00    & 99.95    & 84.65                & 100.00                   & 99.94                    \\
                        & 10                         & 65.57   & 19.45   & 99.36    & 34.43                & 70.34                    & 98.14                    \\
                        & 20                         & 84.46   & 39.34   & 96.70    & 15.54                & 53.42                    & 78.76                    \\
                        & 40                         & 79.91   & 43.00   & 92.84    & 20.09                & 46.19                    & 64.36                    \\
                        & 60                         & 76.82   & 44.27   & 89.47    & 23.18                & 42.37                    & 54.57                    \\
                        & 80                         & 73.35   & 44.67   & 86.30    & 26.65                & 39.10                    & 48.59                    \\
                        & 100                        & 72.78   & 44.51   & 85.29    & 27.22                & 38.84                    & 45.96                    \\ \midrule
\multirow{9}{*}{TRADES} & 0                          & 11.13   & 0.00    & 100.00   & 88.87                & 100.00                   & 100.00                   \\
                        & 1                          & 87.00   & 47.52   & 98.72    & 13.00                & 45.38                    & 90.15                    \\
                        & 5                          & 80.81   & 52.17   & 95.13    & 19.19                & 35.44                    & 74.62                    \\
                        & 10                         & 78.41   & 52.83   & 93.24    & 21.59                & 32.62                    & 68.69                    \\
                        & 20                         & 74.65   & 51.93   & 89.32    & 25.35                & 30.44                    & 57.87                    \\
                        & 40                         & 70.60   & 50.00   & 84.91    & 29.40                & 29.18                    & 48.67                    \\
                        & 60                         & 68.01   & 48.92   & 82.13    & 31.99                & 28.07                    & 44.14                    \\
                        & 80                         & 66.67   & 48.39   & 80.65    & 33.33                & 27.42                    & 41.94                    \\
                        & 100                        & 65.29   & 47.25   & 78.91    & 34.71                & 27.63                    & 39.24                    \\ \bottomrule
\end{tabular}
\end{table}

\clearpage
\section{Proofs}
\label{app:proofs}

In this section, we provide the proofs of our theoretical results in Section~\ref{sec:hyp-risk} and Appendix~\ref{app:linfnet-details}.

\subsection{Proof of Theorem~\ref{thrm:adv-risk-decomposition}}

\textbf{Theorem~\ref{thrm:adv-risk-decomposition} (restated).}
\emph{
$\cR_{\adv}(f_{\vtheta}, \cD) = \cR_{\nat}(f_{\vtheta}, \cD) + (1 - \cR_{\nat}(f_{\vtheta}, \cD)) \cdot \cR_{\adv}(f_{\vtheta}, \cD_{f_{\vtheta}}^+)$.
}
\begin{proof}
For the sake of brevity, here we use the notation $f$ to represent $f_{\vtheta}$.
Denote by $\vx_{\adv} = \arg \max_{\| \vx' - \vx \| \leq \epsilon} \ind (f_{\vtheta}(\vx') \neq y)$, we have
\begin{equation*}
    \begin{split}
        \cR_{\adv}(f, \cD) & = \underset{(\vx, y) \sim \cD}{\bbE}  \left[ \max_{\| \vx' - \vx \| \leq \epsilon} \ind (f( \vx') \neq y) \right] \\
        & = \underset{(\vx, y) \sim \cD}{\bbE}  \left[\ind (f( \vx_{\adv}) \neq y) \right] \\
        & = \underset{(\vx, y) \sim \cD}{\bbE}  \left[\ind (f( \vx_{\adv}) \neq y) \cdot 1 \right] \\
        & = \underset{(\vx, y) \sim \cD}{\bbE}  \left[\ind (f( \vx_{\adv}) \neq y) \cdot (\ind(f(\vx)=y)+\ind(f(\vx)\neq y)) \right] \\
        & = \underset{(\vx, y) \sim \cD}{\bbE}  \left[\ind (f( \vx_{\adv}) \neq y) \cdot \ind(f(\vx)=y)) \right] + \underset{(\vx, y) \sim \cD}{\bbE}  \left[\ind (f( \vx_{\adv}) \neq y) \cdot \ind(f(\vx) \neq y)) \right] \\
        & = \underset{(\vx, y) \sim \cD}{\bbE}  \left[\ind (f( \vx_{\adv}) \neq y) \cdot \ind(f(\vx)=y)) \right] + \underset{(\vx, y) \sim \cD}{\bbE}  \left[\ind(f(\vx) \neq y)) \right] \\
        & = \underset{(\vx, y) \sim \cD}{\bbE} \left[\ind(f(\vx)=y) \right] \cdot \underset{(\vx, y) \sim \cD_f^+}{\bbE}  \left[\ind (f( \vx_{\adv}) \neq y) \right] + \underset{(\vx, y) \sim \cD}{\bbE}  \left[\ind(f(\vx) \neq y)) \right] \\
        & = \underset{(\vx, y) \sim \cD}{\bbE} \left[1-\ind(f(\vx) \neq y) \right] \cdot \underset{(\vx, y) \sim \cD_f^+}{\bbE}  \left[\ind (f( \vx_{\adv}) \neq y) \right] + \underset{(\vx, y) \sim \cD}{\bbE}  \left[\ind(f(\vx) \neq y)) \right] \\
        & = (1 - \cR_{\nat}(f, \cD)) \cdot \cR_{\adv}(f, \cD_f^+) + \cR_{\nat}(f, \cD) \\
        & = \cR_{\nat}(f, \cD) + (1 - \cR_{\nat}(f, \cD)) \cdot \cR_{\adv}(f, \cD_f^+) \\
    \end{split}
\end{equation*}
\end{proof}

\subsection{Proof of Theorem~\ref{thrm:hyp-risk-decomposition}}

\textbf{Theorem~\ref{thrm:hyp-risk-decomposition} (restated).}
\emph{
$\cR_{\hyp}(f_{\vtheta}, \cD) = 1 - (1 - \cR_{\hyp}(f_{\vtheta}, \cD_{f_{\vtheta}}^-)) \cdot \cR_{\nat}(f_{\vtheta}, \cD)$.
}
\begin{proof}
For the sake of brevity, here we use the notation $f$ to represent $f_{\vtheta}$.
Denote by $\vx_{\hyp} = \arg \max_{\| \vx' - \vx \| \leq \epsilon} \ind (f(\vx')=y)$, we have
\begin{equation*}
    \begin{split}
        \cR_{\hyp}(f, \cD) & = \underset{(\vx, y) \sim \cD}{\bbE}  \left[ \max_{\| \vx' - \vx \| \leq \epsilon} \ind (f( \vx') = y) \right] \\
        & = \underset{(\vx, y) \sim \cD}{\bbE}  \left[\ind (f( \vx_{\hyp}) = y) \right] \\
        & = 1 - \underset{(\vx, y) \sim \cD}{\bbE}  \left[\ind (f( \vx_{\hyp}) \neq y) \right] \\
        & = 1 - \underset{(\vx, y) \sim \cD}{\bbE}  \left[\ind (f( \vx_{\hyp}) \neq y) \cdot 1 \right] \\
        & = 1 - \underset{(\vx, y) \sim \cD}{\bbE}  \left[\ind (f( \vx_{\hyp}) \neq y) \cdot (\ind(f(\vx)=y)+\ind(f(\vx)\neq y)) \right] \\
        & = 1 - \underset{(\vx, y) \sim \cD}{\bbE}  \left[\ind (f( \vx_{\hyp}) \neq y) \cdot \ind(f(\vx)\neq y) \right] \\
        & = 1 - \underset{(\vx, y) \sim \cD}{\bbE} \left[\ind(f(\vx)\neq y\right] \cdot \underset{(\vx, y) \sim \cD_f^-}{\bbE}  \left[\ind (f( \vx_{\hyp}) \neq y) \right] \\
        & = 1 - \underset{(\vx, y) \sim \cD}{\bbE} \left[\ind(f(\vx)\neq y\right] \cdot \underset{(\vx, y) \sim \cD_f^-}{\bbE}  \left[1 - \ind (f( \vx_{\hyp}) = y) \right] \\
        & = 1 - \underset{(\vx, y) \sim \cD}{\bbE} \left[\ind(f(\vx)\neq y\right] \cdot \underset{(\vx, y) \sim \cD_f^-}{\bbE}  \left[1 - \ind (f( \vx_{\hyp}) = y) \right] \\
        & = 1 - \cR_{\nat}(f, \cD) \cdot (1 - \cR_{\hyp}(f, \cD_f^-)) \\
        & = 1 - (1 - \cR_{\hyp}(f, \cD_f^-)) \cdot \cR_{\nat}(f, \cD) \\
    \end{split}
\end{equation*}
\end{proof}

\subsection{Proof of Theorem~\ref{thrm:sta-risk-decomposition}}

\begin{lemma}
\label{lemma:inquality}
$\cR_{\hyp}(f_{\vtheta}, \cD_{f_{\vtheta}}^-) \leq \overline{\cR}_{\hyp}(f_{\vtheta}, \cD_{f_{\vtheta}}^-) \leq \cR_{\sta}(f, \cD_{f_{\vtheta}}^-)$, where $\overline{\cR}_{\hyp}(f_{\vtheta}, \cD_{f_{\vtheta}}^-) = \bbE_{(\vx,y)\sim \cD_{f_{\vtheta}}^-} [\ind (f_{\vtheta}(\vx_{\hyp}) \neq f_{\vtheta}(\vx))]$ and $\vx_{\hyp} = \arg \max_{\| \vx' - \vx \| \leq \epsilon} \ind (f_{\vtheta}(\vx')=y)$.
\end{lemma}

\begin{proof}
For the sake of brevity, here we use the notation $f$ to represent $f_{\vtheta}$.
Remember that $f(\vx) \neq y$ for any $(\vx, y) \sim \cD_f^-$.
For the first inequality $\cR_{\hyp}(f, \cD_f^-) \leq \overline{\cR}_{\hyp}(f, \cD_f^-)$, we have
\begin{equation*}
    \begin{split}
        \cR_{\hyp}(f, \cD_f^-) & = \underset{(\vx, y) \sim \cD_f^-}{\bbE}  \left[ \max_{\| \vx' - \vx \| \leq \epsilon} \ind (f( \vx') = y) \right] \\
        & = \underset{(\vx, y) \sim \cD_f^-}{\bbE}  \left[\ind (f( \vx_{\hyp}) = y) \right] \\
        & \leq \underset{(\vx, y) \sim \cD_f^-}{\bbE}  \left[\ind (f( \vx_{\hyp}) \neq f(\vx)) \right], \\
    \end{split}
\end{equation*}
where the above inequality involves two conditions:
\begin{equation*}
    \begin{aligned}
        \ind (f(\vx_{\hyp}) = y) =\left\{
        \begin{array}{ll}
            1 \quad = \quad \ind (f( \vx_{\hyp}) \neq f( \vx)), & \text { if } f\left(\vx_{\hyp}\right) = y, \\
            0 \quad \leq \quad \ind (f( \vx_{\hyp}) \neq f( \vx)), & \text { if } f\left(\vx_{\hyp}\right) \neq y.
        \end{array}\right.
    \end{aligned}
\end{equation*}
For the second inequality $\overline{\cR}_{\hyp}(f, \cD_f^-) \leq \cR_{\sta}(f, \cD_f^-)$, we have
\begin{equation*}
    \begin{split}
        \overline{\cR}_{\hyp}(f, \cD_f^-) & = \underset{(\vx, y) \sim \cD_f^-}{\bbE} [\ind (f(\vx_{\hyp}) \neq f(\vx))] \\
        & \leq \underset{(\vx, y) \sim \cD_f^-}{\bbE} \left[ \max_{\| \vx' - \vx \| \leq \epsilon} \ind (f( \vx') \neq f(\vx))\right], \\
    \end{split}
\end{equation*}
where the above inequality involves two conditions:
\begin{equation*}
    \begin{aligned}
        \ind (f(\vx_{\hyp}) \neq f(\vx)) =\left\{
        \begin{array}{ll}
            1 \quad = \quad \max_{\| \vx' - \vx \| \leq \epsilon} \ind (f( \vx') \neq f(\vx)), & \text { if } f\left(\vx_{\hyp}\right) = y, \\
            0 \quad \leq \quad \max_{\| \vx' - \vx \| \leq \epsilon} \ind (f( \vx') \neq f(\vx)), & \text { if } f\left(\vx_{\hyp}\right) \neq y.
        \end{array}\right.
    \end{aligned}
\end{equation*}
\end{proof}

\noindent
\textbf{Theorem~\ref{thrm:sta-risk-decomposition} (restated).}
\emph{
$\cR_{\sta}(f_{\vtheta}, \cD) = (1 - \cR_{\nat}(f_{\vtheta}, \cD)) \cdot \cR_{\adv}(f_{\vtheta}, \cD_{f_{\vtheta}}^+) + \cR_{\nat}(f_{\vtheta}, \cD) \cdot \cR_{\sta}(f_{\vtheta}, \cD_{f_{\vtheta}}^-)$, where we have $\cR_{\sta}(f_{\vtheta}, \cD_{f_{\vtheta}}^-) \ge \cR_{\hyp}(f_{\vtheta}, \cD_{f_{\vtheta}}^-)$.
}
\begin{proof}
For the sake of brevity, here we use the notation $f$ to represent $f_{\vtheta}$.
\begin{equation*}
    \begin{split}
        \cR_{\sta}(f, \cD) & = \underset{(\vx, y) \sim \cD}{\bbE}  \left[ \max_{\| \vx' - \vx \| \leq \epsilon} \ind (f( \vx') \neq f(\vx)) \right] \\
        & = \underset{(\vx, y) \sim \cD}{\bbE}  \left[ \max_{\| \vx' - \vx \| \leq \epsilon} \ind (f( \vx') \neq f(\vx)) \cdot 1 \right] \\
        & = \underset{(\vx, y) \sim \cD}{\bbE}  \left[ \max_{\| \vx' - \vx \| \leq \epsilon} \ind (f( \vx') \neq f(\vx)) \cdot (\ind(f(\vx)=y)+\ind(f(\vx)\neq y)) \right] \\
        & = \underset{(\vx, y) \sim \cD}{\bbE}  \left[ \max_{\| \vx' - \vx \| \leq \epsilon} \ind (f( \vx') \neq f(\vx)) \cdot \ind(f(\vx)=y) \right] \\
        & \quad + \underset{(\vx, y) \sim \cD}{\bbE}  \left[ \max_{\| \vx' - \vx \| \leq \epsilon} \ind (f( \vx') \neq f(\vx)) \cdot \ind(f(\vx) \neq y) \right] \\
        & = \underset{(\vx, y) \sim \cD}{\bbE}  \left[ \max_{\| \vx' - \vx \| \leq \epsilon} \ind (f( \vx') \neq y) \cdot \ind(f(\vx)=y) \right] \\
        & \quad + \underset{(\vx, y) \sim \cD}{\bbE}  \left[ \max_{\| \vx' - \vx \| \leq \epsilon} \ind (f( \vx') \neq f(\vx)) \cdot \ind(f(\vx) \neq y) \right] \\
        & = \underset{(\vx, y) \sim \cD}{\bbE} \left[\ind(f(\vx)=y)\right] \cdot \underset{(\vx, y) \sim \cD_f^+}{\bbE} \left[ \max_{\| \vx' - \vx \| \leq \epsilon} \ind (f( \vx') \neq y) \right] \\
        & \quad + \underset{(\vx, y) \sim \cD}{\bbE} \left[\ind(f(\vx) \neq y)\right] \cdot \underset{(\vx, y) \sim \cD_f^-}{\bbE}  \left[ \max_{\| \vx' - \vx \| \leq \epsilon} \ind (f( \vx') \neq f(\vx))\right] \\
        & = \underset{(\vx, y) \sim \cD}{\bbE} \left[1-\ind(f(\vx)=y)\right] \cdot \underset{(\vx, y) \sim \cD_f^+}{\bbE} \left[ \max_{\| \vx' - \vx \| \leq \epsilon} \ind (f( \vx') \neq y) \right] \\
        & \quad + \underset{(\vx, y) \sim \cD}{\bbE} \left[\ind(f(\vx) \neq y)\right] \cdot \underset{(\vx, y) \sim \cD_f^-}{\bbE}  \left[ \max_{\| \vx' - \vx \| \leq \epsilon} \ind (f( \vx') \neq f(\vx))\right] \\
        & = (1 - \cR_{\nat}(f, \cD)) \cdot \cR_{\adv}(f, \cD_f^+) + \cR_{\nat}(f, \cD) \cdot \cR_{\sta}(f, \cD_f^-) \\
    \end{split}
\end{equation*}
Then, by combining Lemma~\ref{lemma:inquality}, we have $\cR_{\sta}(f_{\vtheta}, \cD_{f_{\vtheta}}^-) \ge \cR_{\hyp}(f_{\vtheta}, \cD_{f_{\vtheta}}^-)$.
\end{proof}

\subsection{Proof of Corollary~\ref{coro:certify-hyp-risk}}
\label{app:proof-certify-hyp-risk}

 \citet{zhang2021towards} showed that $\linf$-dist nets are 1-Lipschitz with respect to $\linf$-norm, which is a nice theoretical property in controlling the robustness of the model.
 
\begin{fact} (Fact 3.4 in~\citet{zhang2021towards})
\label{fact:1lipschitz}
Any $\linf$-dist net $\vg(\cdot)$ is 1-Lipschitz with respect to $\linf$-norm, i.e., for any $\vx_1, \vx_2 \in \bbR^d$, we have $\| \vg(\vx_1) - \vg(\vx_2) \|_{\infty} \leq \| \vx_1 - \vx_2 \|_{\infty}$.
\end{fact}

Since $\vg$ is 1-Lipschitz with respect to $\linf$-norm, if the perturbation over $\vx$ is rather small, the change of the output can be bounded and the predication of the perturbed data $\vx'$ will not change to $y$ as long as $\max_i g_i(\vx') > g_y(\vx')$, which directly bounds the certified hypocritical risk.

\noindent
\textbf{Corollary~\ref{coro:certify-hyp-risk} (restated).}
\emph{
Given the $\linf$-dist net $f(\vx) = \arg \max_{i\in [M]} g_i(\vx)$ defined in Eq.~\ref{eq:linf-net} and the distribution of misclassified examples $\cD_f^-$ with respect to $f$, we have
$$
\textstyle
\cR_{\hyp}(f, \cD_f^-) \leq \bbE_{(\vx,y)\sim \cD_f^-} \left[\ind (\max_i g_i(\vx) - g_y(\vx) < 2 \epsilon) \right].
$$
}
\begin{proof}
Firstly, similar to the robust radius~\citep{zhai2020macer, zhang2021towards}, we define the \textit{hypocritical radius} as:
\begin{equation*}
    \mathrm{HR}(f; \vx, y) = \inf_{f(\vx') = y} \| \vx'-\vx \|_{\infty},
\end{equation*}
which represents the minimal radius required to correct model prediction.
Exactly computing the the hypocritical radius of a classifier induced by a standard deep neural network is very difficult. Next we seek to derive a tight lower bound of $\mathrm{HR}(f; \vx, y)$ using the $\linf$-dist net. Considering a misclassified example $\vx$ such that $f(\vx) \neq y$, we define $\mathrm{margin}(f; \vx, y)$ as the difference between the largest and the $y$-th elements of $\vg(\vx)$:
\begin{equation*}
    \mathrm{margin}(f; \vx, y) = \max_i g_i(\vx) - g_y(\vx).
\end{equation*}
Then for any $\vx'$ satisfying $\|\vx-\vx'\|_{\infty} < \mathrm{margin}(f; \vx, y)/2$, each element of $\vg(\vx)$ can move at most $\mathrm{margin}(f; \vx, y)/2$ when $\vx$ changes to $\vx'$ due to that $\vg(\vx)$ is 1-Lipschitz. Therefore, $g_y(\vx')$ cannot be the largest element of $\vg(\vx)$, that is, $f(\vx') \neq y$.
In other words, $\mathrm{margin}(f; \vx, y)/2$ is a provable lower bound of the hypocritical radius: $\mathrm{margin}(f; \vx, y)/2 \leq \mathrm{HR}(f; \vx, y)$.

This certified radius leads to a guaranteed upper bound of the hypocritical risk:
\begin{equation*}
    \max_{\| \vx'-\vx \|_{\infty} \leq \epsilon} \ind (f(\vx')=y) \leq \ind(\mathrm{margin}(f; \vx, y)/2 < \epsilon),
\end{equation*}
i.e., an example can be perturbed to be classified only if the difference between $\max_i g_i(\vx)$ and $g_y(\vx)$ is less than $\epsilon$. The expectation of the above inequality over $\cD_f^-$ serves as a performance metric of the provable hypocritical robustness:
\begin{equation*}
    \underset{(\vx, y) \sim \cD_f^-}{\bbE} \left[ \max_{\| \vx'-\vx \|_{\infty} \leq \epsilon} \ind (f(\vx')=y) \right]
    \leq
    \underset{(\vx, y) \sim \cD_f^-}{\bbE} \left[\ind (\max_i g_i(\vx) - g_y(\vx) < 2 \epsilon) \right].
\end{equation*}
\end{proof}

\section{Crafting Hypocritical Data via Various Attacks}
\label{app:various_attacks}

In addition to the PGD attack adopted in the main text, many other techniques (such as FGSM~\citep{goodfellow2014explaining} and C\&W~\citep{carlini2017towards}) that were developed for crafting adversarial examples can also be used for the current study of hypocritical attacks. To support this claim, we adaptively modify their objective function to craft hypocritical examples. Results on CIFAR-10 are summarized in Table~\ref{tab:various_attacks}, where PGD-20/100 denotes the PGD attack with 20/100 iterations, and C\&W-100 is similar. We observe that for hypocritical attacks: \textit{i)} FGSM is relatively weak, \textit{ii)} PGD-100 is slightly stronger than PGD-20, and \textit{iii)} C\&W is slightly better than PGD. These phenomena are similar to the case of adversarial examples, indicating that future work may easily adapt more advanced techniques, such as AutoAttack~\citep{croce2020reliable}, to further boost hypocritical attacks.

\begin{table}[!h]
\footnotesize
\centering
\vspace{-4px}
\caption{
Verification accuracy (\%) of ResNet-18 models under $\linf$ threat model across various attack techniques. Here, the THRM model is trained with $\lambda=40$, whose clean accuracy is comparable to the TRADES model with $\lambda=6$.
}
\vspace{-4px}
\label{tab:various_attacks}
\begin{tabular}{@{}lrrrrr@{}}
\toprule
Model            & Clean & FGSM  & PGD-20          & PGD-100         & C\&W-100          \\ \midrule
Naive            & 8.46  & 28.29 & 82.56           & 84.24           & \textbf{88.23}  \\
Quality (NT)     & 94.38 & 99.60 & \textbf{100.00} & \textbf{100.00} & \textbf{100.00} \\
Quality (PGD-AT) & 84.08 & 88.96 & 96.92           & 96.95           & \textbf{97.29}  \\
Quality (TRADES) & 81.05 & 86.28 & 95.17           & 95.19           & \textbf{96.08}  \\
Quality (THRM)   & 81.19 & 86.10 & 94.18           & 94.20           & \textbf{95.13}  \\ \bottomrule
\end{tabular}
\vspace{-5px}
\end{table}

\section{Trade-off between Risks}
\label{app:toy-tradeoffs}

\subsection{Example 1}

Motivated by the trade-off between natural risk and adversarial risk~\citep{tsipras2018robustness, zhang2019theoretically}, we notice that there may also exist a tension between the goal of natural risk minimization and hypocritical risk minimization. To illustrate the phenomenon, similar to the toy example in~\citet{zhang2019theoretically}, we provide another toy example here.

\begin{figure}[h]
\vspace{-8px}
    \centering
    \includegraphics[width=0.45\textwidth]{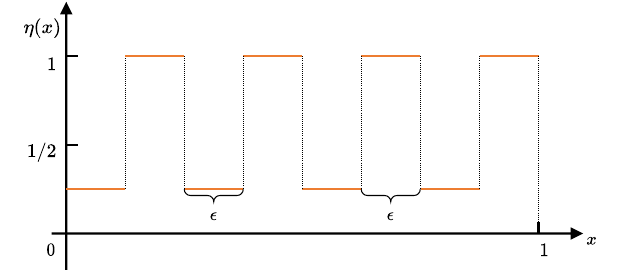}
    \caption{Visualization of $\eta(x)$.}
    \label{fig:toy-exp1}
\end{figure}

Consider the case $(x, y) \in \bbR \times \{-1, +1\}$ from a distribution $\cD$, where the marginal distribution over the instance space is a uniform distribution over $[0, 1]$, and for $k=0,1,\cdots, \lceil\frac{1}{2\epsilon} - 1\rceil$,
\begin{equation*}
    \begin{split}
        \eta(x) & = \operatorname{Pr}(y=+1 \mid x) \\
        & = \left\{
            \begin{array}{ll}
                1/4, & x \in[2 k \epsilon,(2 k+1) \epsilon), \\
                1, & x \in((2 k+1) \epsilon,(2 k+2) \epsilon].
            \end{array}\right.
    \end{split}
\end{equation*}
See Figure~\ref{fig:toy-exp1} for visualization of $\eta(x)$. We consider two classifiers: a) the Bayes optimal classifier $\operatorname{sign}(2\eta(x) - 1)$; b) the all-one classifier which always outputs ``positive". Table~\ref{tab:toy-exp1-tradeoff} displays the trade-off between the natural risk and the hypocritical risk on misclassified examples: the minimal natural risk is achieved by the Bayes optimal classifier with large hypocritical risk, while the minimal hypocritical risk is achieved by the all-one classifier with large natural risk.

\begin{table}[h]
\small
   \caption{Comparison of Bayes optimal classifier and all-one classifier.}
   \label{tab:toy-exp1-tradeoff}
   \vspace{-5px}
   \begin{center}
      \begin{tabular}{lcc}
         \toprule
         & Bayes Optimal Classifier & All-One Classifier \\
         \midrule
         $\cR_{\nat}(f, \cD)$ & $1/8$ (optimal) & $3/8$ \\
         $\cR_{\hyp}(f, \cD)$ & $1$ & $5/8$ \\
         $\cR_{\adv}(f, \cD)$ & $1$ & $3/8$ \\
         $\cR_{\hyp}(f, \cD_f^-)$ & $1$ & $0$ (optimal) \\
         $\cR_{\adv}(f, \cD_f^+)$ & $1$ & $0$ (optimal) \\
         \bottomrule
      \end{tabular}
   \end{center}
\end{table}

\subsection{Example 2}

In Example~1, the optimal solution for adversarial risk and hypocritical risk is the same by coincidence. Next, we show that the minimizers of natural risk and hypocritical risk are not necessarily consistent by providing the second toy example, which is motivated by the well-known trade-off between precision and recall~\citep{buckland1994relationship}.

Consider the case $(\vx, y) \in \bbR^2 \times \{-1, +1\}$ from a distribution $\cD$, where the marginal distribution over the instance space is a uniform distribution over $[0, 1]^2$. Let the decision boundary of the oracle (ground truth) be a circle:
\begin{equation*}
    \mathcal{O}(\vx) = \operatorname{sign}(r - \|\vx-\vc\|_2),
\end{equation*}
where the center $\vc=(0.5, 0.5)^\top$ and the radius $r=0.4$. The points inside the circle are labeled as belonging to the positive class, otherwise they are labeled as belonging to the negative class. We consider the linear classifier $f$ with fixed $\vw = (0, 1)^\top$ and a tunable threshold $b$:
\begin{equation*}
    f(\vx) = \operatorname{sign}(\vw^\top \vx - b) = \operatorname{sign}(x_2 - b).
\end{equation*}

\begin{figure}[t]
\vspace{-8px}
    \centering
    \includegraphics[width=0.5\textwidth]{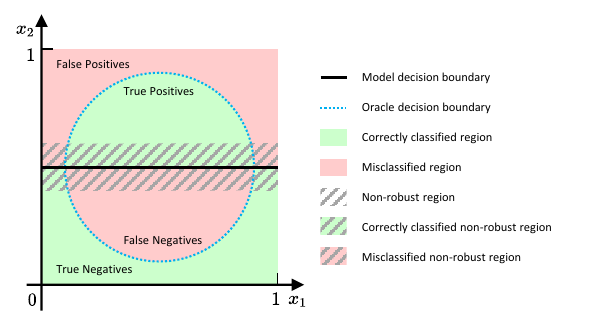}
\vspace{-8px}
    \caption{Visualization of decision boundary for oracle and model.}
    \label{fig:toy-exp2-decision_boundary}
\end{figure}

\begin{figure}[t]
    \centering
    \includegraphics[width=0.6\textwidth]{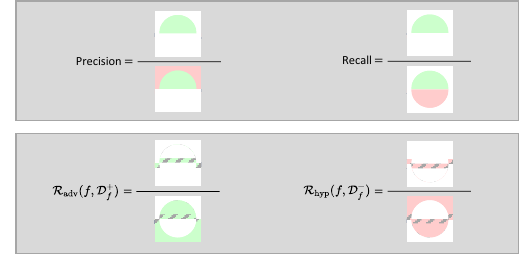}
    \caption{Visualization of formulation for precision, recall, adversarial risk, and hypocritical risk. These metrics can be viewed as the ratio of specific areas in the section example.}
    \label{fig:toy-exp2-risk_vis}
\end{figure}

\begin{figure}[!h]
   \centering
   \begin{minipage}{.4\textwidth}
      \begin{center}
         \includegraphics[width=1.0\textwidth]{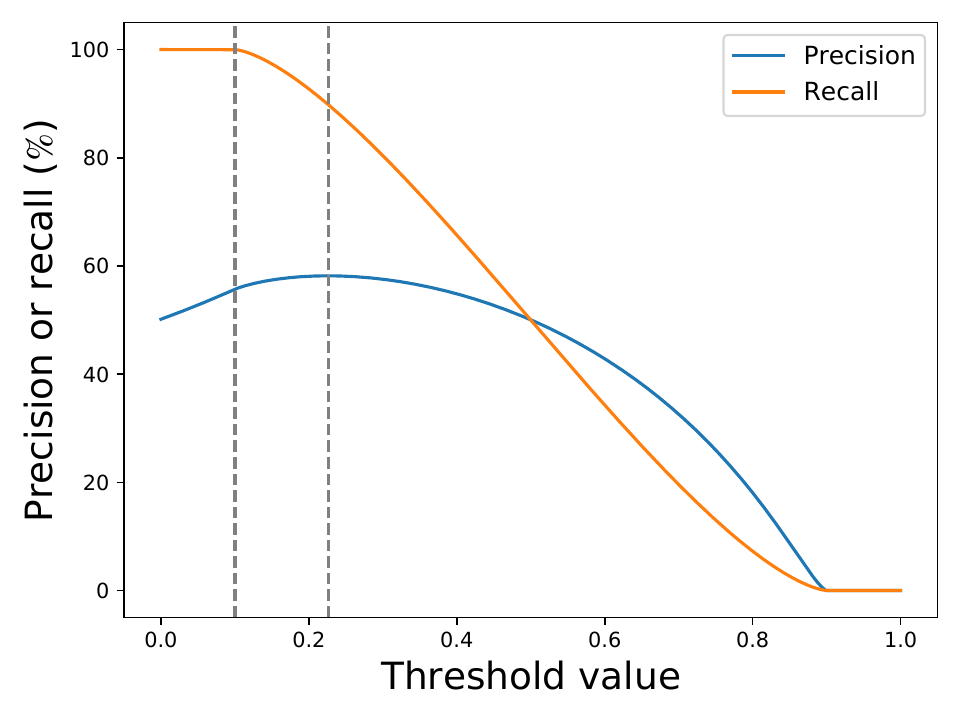}
      \end{center}
      \caption{The tradeoff between precision and recall in the section example. }
      \label{fig.tradeoff.precision_recall_curve}
   \end{minipage}%
   \quad \quad
   \begin{minipage}{.4\textwidth}
      \begin{center}
         \includegraphics[width=1.0\textwidth]{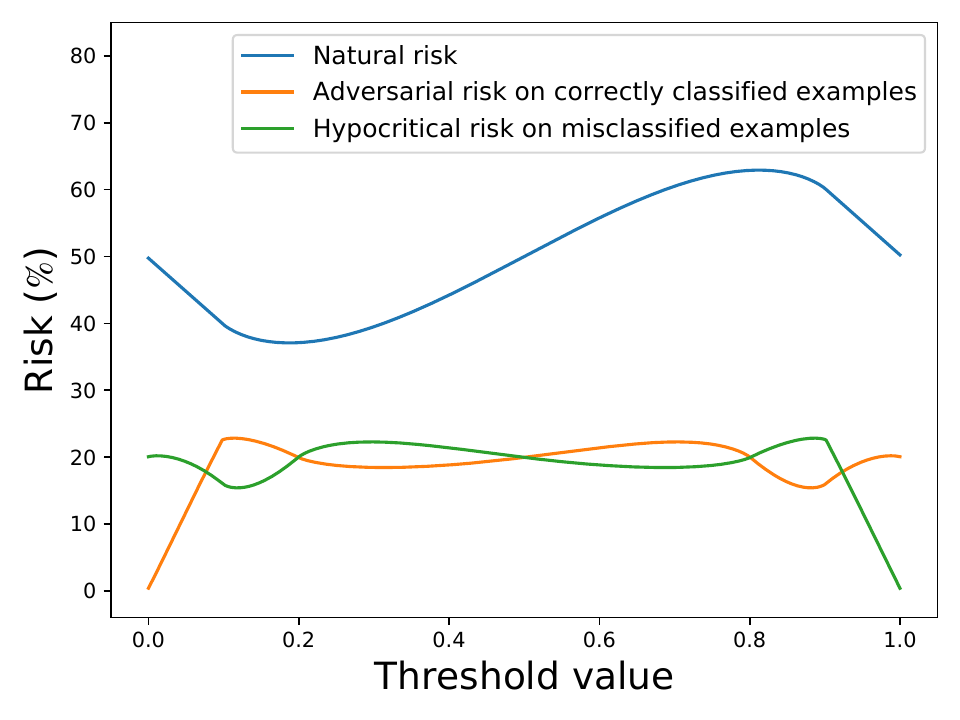}
      \end{center}
      \caption{The tradeoff between risks in the section example. }
      \label{fig.tradeoff.adv_hyp_curve}
   \end{minipage}
\end{figure}

See Figure~\ref{fig:toy-exp2-decision_boundary} for visualization of decision boundary for the oracle and the linear classifier over the instance space. Here we choose the $l_2$-norm with $\epsilon = 0.1$ as the threat model. We tune the threshold $b$ of the classifier over $[0, 1]$ to show the trade-off between metrics. The computation of the metrics is visualized in Figure~\ref{fig:toy-exp2-risk_vis}. Precision is the number of true positives divided by the the sum of true positives and false positives. Recall is the number of true positives divided by the sum of true positives and false negatives. We compare the adversarial risk on correctly classified examples and the hypocritical risk on misclassified examples.

Figure \ref{fig.tradeoff.precision_recall_curve} plots the curve of precision and recall versus the threshold $b$. As we can see, there is a explicit precision-recall tradeoff between the two gray dotted lines. Similarly, Figure \ref{fig.tradeoff.adv_hyp_curve} plots the curve of $\cR_{\adv}(f, \cD_f^+)$ and $\cR_{\hyp}(f, \cD_f^-)$ versus the threshold $b$. As we can see, the tradeoff exists almost everywhere: as the adversarial risk increases, the hypocritical risk decreases, and vise versa.

\clearpage
\section{Details on THRM}
\label{app:thrm-details}

In this section, we derive the objective of THRM in Eq.~(\ref{eq:thrm}) from the theoretical perspective.

As a by-product of Lemma~\ref{lemma:inquality}, we get a tighter upper bound of the hypocritical risk on misclassified examples, i.e., $\overline{\cR}_{\hyp}(f_{\vtheta}, \cD_{f_{\vtheta}}^-)$ defined in Lemma~\ref{lemma:inquality}. Similar to TRADES, we propose to minimize the composition of the natural risk and this tighter upper bound: $\cR_{\nat}(f_{\vtheta}, \cD) + \lambda \overline{\cR}_{\hyp}(f_{\vtheta}, \cD_{f_{\vtheta}}^-)$.
However, optimization over the 0-1 loss is intractable in practice.
Inspired by~\citet{madry2018towards, zhang2019theoretically, wang2019improving}, we also seek for proper surrogate loss functions which are computationally tractable. Specifically, we adopt the commonly used cross entropy (CE) loss as the surrogate loss for the indicator function $\ind (f_{\vtheta}(\vx) \neq y)$ in the natural risk. Further, observed that $\overline{\cR}_{\hyp}(f_{\vtheta}, \cD_{f_{\vtheta}}^-) = (1 / \cR_{\nat}(f_{\vtheta}, \cD)) \cdot \overline{\cR}_{\hyp}(f_{\vtheta}, \cD)$, we absorb the term $\cR_{\nat}(f_{\vtheta}, \cD)$ into $\lambda$, and similarly use the KL divergence as the surrogate loss~\citep{zheng2016improving, zhang2019theoretically, wang2019improving} for the indicator function $\ind (f_{\vtheta}(\vx_{\hyp}) \neq f_{\vtheta}(\vx))$.
Based on the surrogate loss functions, we can state the final objective function for the adaptive robust training method:
\begin{equation*}
\textstyle
    \cL_{\thrm} = \underset{(\vx, y) \sim \cD}{\bbE}  \left[ \ce(\vp_{\vtheta}(\vx), y) + \lambda \cdot \kl(\vp_{\vtheta}(\vx) || \vp_{\vtheta}(\vx_{\hyp})) \right],
\end{equation*}
where $\vx_{\hyp}$ is approximately generated by minimizing the CE loss using PGD as in Section~\ref{sec:hyp-examples}. The first term in the above objective is the standard loss that maximizes the natural accuracy. The second term is also physically meaningful, which forces the predictive distributions of clean examples and hypocritical examples to be similar to each other, and thus encourages the model to robustly predict its failures on hypocritical examples.

We name the method THRM (Trade-off for Hypocritical Risk Minimization). This is largely because we notice that, similar to the trade-off between the natural risk and the adversarial risk~\citep{tsipras2018robustness, zhang2019theoretically}, the minimizers of the natural risk and the hypocritical risk are not necessarily consistent. Training models to be robust against hypocritical perturbations may lead to a reduction of natural accuracy. We illustrated this phenomenon by constructing toy examples in Appendix~\ref{app:toy-tradeoffs}.

\section{Details on $\linf$-dist Nets}
\label{app:linfnet-details}

In Section~\ref{sec:countermeasures-exps}, we empirically computed the hypocritical risk achieved by the PGD-based attack, which is actually a lower bound of the true hypocritical risk. This means that more hypocritical examples may be constructed by stronger hypocritical attacks. Unfortunately, exactly computing the true hypocritical risk for a standard DNN is very difficult, since computing the $\ell_p$ robust radius of a DNN is NP-hard~\citep{katz2017reluplex}. Therefore, we turn to seek for a tight upper bound of the true hypocritical risk. Similar to the definition of certified adversarial risk in previous works~\citep{raghunathan2018certified, wong2018provable, weng2018towards, zhang2018efficient, gowal2019scalable, salman2019convex, boopathy2019cnn, dathathri2020enabling, xu2020fast}, we define the \textit{certified hypocritical risk} as a provable upper bound on the achievable hypocritical risk by \textit{any} hypocritical attack. 

Next, we exemplify the certified hypocritical risk using a recent state-of-the-art certification method called $\linf$-dist nets~\citep{zhang2021towards}, which inherently resist $\linf$ perturbations by using $\linf$-distance as the basic operation.
An $L$ layer $\linf$-dist net is defined as $f(\vx) = \arg \max_{i\in [M]} g_i(\vx)$, in which
\begin{equation}
\label{eq:linf-net}
    \textstyle
    \vg(\vx) = (-x_1^{(L)}, -x_2^{(L)}, \ldots, -x_M^{(L)}),
    x_k^{(l)} = \| \vx^{l-1} - \vw^{(l,k)} \|_{\infty} + b^{(l,k)}, 1 \leq l \leq L, 1 \leq k \leq d_l,
\end{equation}
where $d_l$ is the number of units in the $l$-th layer, $\vx^{(0)}=\vx$ is the input, $\vx^{l}=(x_1^{(l)}, x_2^{(l)}, \cdots, x_{d_l}^{(l)})$ is the output of the $l$-th layer, $\vw^{(l,k)}$ and $b^{(l,k)}$ is the parameters of the $k$-th unit in the $l$-th layer.

The following corollary allows us to make use of the $\linf$-dist net to certify the hypocritical risk:
\begin{corollary}
\label{coro:certify-hyp-risk}
Given the $\linf$-dist net $f(\vx) = \arg \max_{i\in [M]} g_i(\vx)$ defined in Eq.~\ref{eq:linf-net} and the distribution of misclassified examples $\cD_f^-$ with respect to $f$, we have
$$
\textstyle
\cR_{\hyp}(f, \cD_f^-) \leq \bbE_{(\vx,y)\sim \cD_f^-} \left[\ind (\max_i g_i(\vx) - g_y(\vx) < 2 \epsilon) \right].
$$
\end{corollary}
The proof is provided in Appendix~\ref{app:proof-certify-hyp-risk}, which utilizes the 1-Lipschitz property of $\linf$-dist nets to guarantee the upper bound on the hypocritical risk.
Using this bound, we can certify the hypocritical risk of an $\linf$-dist net under $\linf$-norm perturbations with little computational cost (only one forward pass for each example).
Following the implementation in~\citet{zhang2021towards}, we use a 6-layer $\linf$-dist net for CIFAR-10. Other configurations and hyper-parameters are also the same as the original paper.

Results of the models trained on the \textsl{Quality} data and the \textsl{Poisoning} data are summarized in Table~\ref{tab:linf-net-results}.
We report both the empirical hypocritical risk achieved by the PGD-based attack and the certified hypocritical risk guaranteed by Corollary~\ref{coro:certify-hyp-risk}. The former is naturally lower than the latter. The true hypocritical risk will exist between the empirical and certified hypocritical risks.
For completeness, we additionally report the certified adversarial (hypocritical) accuracy, which is a strict lower bound (upper bound) on the achievable accuracy by any adversarial (hypocritical) attack.

As can be seen from Table~\ref{tab:linf-net-results}, the $\linf$-dist net achieves moderate certified hypocritical risk. For both the \textsl{Quality} model and the \textsl{Poisoning} model, nearly half of the errors are guaranteed not to be covered up by any attack.
On the other hand, the $\linf$-dist net seems to perform worse than TRADES and THRM in terms of empirical hypocritical risk.
For example, when training on the CIFAR-10 \textsl{Quality} data, the $\linf$-dist net has 43.34\% natural risk and 47.07\% empirical hypocritical risk; for comparison, TRADES with $\lambda=80$ achieves 33.47\% natural risk and 40.99\% empirical hypocritical risk.
Thus, the certified robustness against hypocritical perturbations is barely satisfactory;
in particular, the PGD-based attack can already conceal almost half of the errors, and perhaps more powerful attacks can conceal more than half of the errors.
Other advanced techniques (such as designing more effective architectures~\citep{huang2021exploring}, more principled initializations~\citep{shi2021fast}, and combination with other certification methods~\citep{zhang2019towards}) may help to achieve a better certified hypocritical risk.
We leave this as future work.

\begin{table*}[t]
\small
\centering
\vspace{-5px}
\caption{
Empirical and certified accuracy (\%) / risk (\%) of $\linf$-dist nets on CIFAR-10 under $\linf$ threat model.
$\cD$, $\cA$, and $\cF$ denote the model accuracy evaluated on clean examples, adversarially perturbed examples, and hypocritically perturbed examples, respectively.
$\cR_{\nat}$ and $\cR_{\hyp}$ denote the natural risk and the hypocritical risk on misclassified examples, respectively.
}
\label{tab:linf-net-results}
\begin{tabular}{@{}l|ccccc|ccc@{}}
\toprule
\multirow{2}{*}{Model} &
  \multirow{2}{*}{$\cD$} &
  \multicolumn{2}{c}{$\cA$} &
  \multicolumn{2}{c|}{$\cF$} &
  \multirow{2}{*}{$\cR_{\nat}$} &
  \multicolumn{2}{c}{$\cR_{\hyp}$} \\ \cmidrule(lr){3-4} \cmidrule(lr){5-6} \cmidrule(l){8-9} 
          &       & Empirical  & Certified  & Empirical  & Certified  &       & Empirical  & Certified  \\ \midrule
Poisoning & 55.62 & 36.36 & 31.87 & 76.57 & 83.36 & 44.38 & 47.25 & 62.51 \\
Quality   & 56.66 & 37.24 & 32.79 & 77.06 & 83.86 & 43.34 & 47.07 & 62.76 \\ \bottomrule
\end{tabular}
\vspace{-3px}
\end{table*}

\end{document}